\documentclass[twoside,11pt]{article}

 \usepackage[abbrvbib, preprint]{jmlr2e}

\usepackage{jmlr2e}

\usepackage{amsmath,amsfonts,bm}

\def\eqref#1{equation~\ref{#1}}

\def\1{\bm{1}}

\def\vh{{\bm{h}}}

\def\vs{{\bm{s}}}

\def\vu{{\bm{u}}}
\def\vv{{\bm{v}}}
\def\vw{{\bm{w}}}
\def\vx{{\bm{x}}}
\def\vy{{\bm{y}}}

\def\mA{{\bm{A}}}

\def\mD{{\bm{D}}}

\def\mI{{\bm{I}}}

\def\mP{{\bm{P}}}

\def\mS{{\bm{S}}}
\def\mT{{\bm{T}}}

\def\mW{{\bm{W}}}

\DeclareMathAlphabet{\mathsfit}{\encodingdefault}{\sfdefault}{m}{sl}
\SetMathAlphabet{\mathsfit}{bold}{\encodingdefault}{\sfdefault}{bx}{n}

\def\gH{{\mathcal{H}}}
\def\gI{{\mathcal{I}}}

\def\gX{{\mathcal{X}}}
\def\gY{{\mathcal{Y}}}

\newcommand{\E}{\mathbb{E}}

\newcommand{\R}{\mathbb{R}}

\usepackage{hyperref}
\usepackage{url}
\usepackage{amsmath}
\usepackage{graphicx}

\usepackage{multirow}
\usepackage{subfig}
\usepackage{wrapfig}
\usepackage{xcolor}

\def\IN{{\mathsf{IN}}}
\def\LN{{\mathsf{LN}}}
\def\BN{{\mathsf{BN}}}
\def\GN{{\mathsf{GN}}}
\def\NO{{\mathsf{NO}}}
\def\RMSn{{\mathsf{RMSnorm}}}

\usepackage{lastpage}

\firstpageno{1}

\begin{document}

\title{The Hidden Power of Normalization Layers  in  Neural Networks: Exponential Capacity Control}

\author{\name Khoat Than \email khoattq@soict.hust.edu.vn \\
       \addr School of Information and Communication Technology, \\
       Hanoi University of Science and Technology, Hanoi, Vietnam
}

\editor{ }

\maketitle

\begin{abstract}
Normalization layers are critical components of modern AI systems, such as ChatGPT, Gemini, DeepSeek, etc. Empirically, they are known to stabilize training dynamics and improve generalization ability. However, the underlying theoretical mechanism by which normalization layers contribute to both optimization and generalization remains largely unexplained, especially when using many normalization layers in a deep neural network (DNN).

In this work, we develop a theoretical framework that elucidates the role of normalization through the lens of capacity control. We prove that an unnormalized DNN can exhibit exponentially large Lipschitz constants with respect to either its parameters or inputs, implying excessive functional capacity and potential overfitting. \textit{Such bad DNNs are uncountably many}. In contrast, \textit{the insertion of normalization layers provably can  reduce the Lipschitz constant at an exponential rate} in the number of normalization layers. This exponential reduction yields two fundamental consequences: (1) it smooths the loss landscape at an exponential rate, facilitating faster and more stable optimization; and (2) it constrains the effective capacity of the network, thereby enhancing generalization guarantees on unseen data. Our results thus offer a principled explanation for the empirical success of normalization methods in deep learning.
\end{abstract}

\begin{keywords}
 Normalization, Lipschitz constant, Capacity control, Generalization, Deep learning
\end{keywords}

\section{Introduction} \label{sec-Introduction}

Normalization methods such as Batch Normalization ($\BN$)~\citep{ioffe2015BNorm}, Layer Normalization ($\LN$)~\citep{ba2016LNorm}, Group Normalization ($\GN$)~\citep{wu2020GNorm},  Instance Normalization ($\IN$)~\citep{ulyanov2016INorm}, and $\RMSn$ \citep{zhang2019RMSnorm} have become critical components of modern AI systems, such as ChatGPT, Gemini, DeepSeek, etc. Among them, $\LN$ or $\RMSn$ are frequently used in  large language models \citep{openai2024gpt4,abdin2024phi4,gu2024mamba,behrouz2025nested} and vision models \citep{peebles2023DiT,ma2024sit,frans25Shortcut}. 

In practice, normalization layers consistently enhance performance of deep neural networks (DNNs) across a wide range of tasks.  A classical motivation for normalization is to mitigate the \emph{internal covariate shift} during training~\citep{ioffe2015BNorm}. Without normalization, the variance of inputs at intermediate layers may grow substantially~\citep{ioffe2015BNorm,bjorck2018understandingBN}---even exponentially with depth~\citep{de2020BN,zhang2019fixup,shao2020normalization}---causing unstable gradients and slow convergence. Normalization counteracts this effect by rescaling inputs to improve optimization stability.

Beyond optimization stability, normalization layers has long been observed to improve the \emph{generalization} of DNNs, yet its underlying theoretical mechanism remains unclear. Several hypotheses have been proposed. \citet{santurkar2018BNorm} showed that $\BN$ smooths the loss landscape by reducing the Lipschitz constant of the loss, which facilitates training. The ability to employ larger learning rates~\citep{bjorck2018understandingBN,zhang2019fixup} and the implicit regularization behavior observed in shallow architectures~\citep{luo2019UnderstandingBN} further suggest a regularizing effect. However, a precise theoretical characterization of this regularization for \emph{deep} networks is still missing~\citep{wu2020WN}. This leads to a long-standing open question:

\begin{quote}
\emph{Why do normalization layers help DNNs generalize well on unseen data?}
\end{quote}

In this work, we answer this question by showing that normalization layers provide \emph{explicit control over the functional capacity} of a network, as quantified by its Lipschitz constant. Under the widely accepted assumption of large input variance, we demonstrate that normalization operations reduce the Lipschitz constant at an \emph{exponential rate} with respect to normalization operations. This exponential reduction simultaneously smooths the optimization landscape and regularizes the network's capacity, leading to improved generalization and faster convergence. Our analysis thus establishes a unified theoretical foundation for the empirical success of normalization methods.

Our main contributions are summarized as follows:

\begin{itemize}
    \item \textit{Lipschitz-based capacity analysis:} 
    We present a comprehensive  analysis of normalization methods through the lens of \emph{Lipschitz-based capacity control}.
    
    \item \textit{Unnormalized networks exhibit extreme Lipschitz behavior:}
    We show that a feedforward ReLU network (FFN) without normalization can have \emph{exponentially large} or \emph{exponentially small} Lipschitz constants (with respect to either weights or inputs), and that such configurations are \emph{uncountably many}. Such constants exponentially depend on the depth. A huge Lipschitz constant w.r.t. the input implies that the network has an excessive functional capacity and hence easily gets overfitting. In contrast, extremely small Lipschitz constant implies that the network is prone to underfitting. Therefore, within one architecture, \textit{there exist uncountably many models which potentially are  either overfitting or underfitting.}
    
   The exponentially  large/small Lipschitz constants w.r.t. weights suggest that the loss landscape is extremely complex. We show that the gradient of the training loss can \textit{vanish (or explode)} at uncountably many weight matrices of the network, posing significant challenges for optimization.
    
    \item \textit{Normalization induces exponential capacity reduction:}
    Under large-variance conditions, each normalization layer ($\BN$, $\LN$, $\GN$, or $\IN$) reduces the Lipschitz constant by a factor inversely proportional to the input variance. Multiple normalization layers thus lead to an \emph{exponential} reduction, providing an explicit theoretical account of their regularization role.
    
    \item \textit{Exponential smoothing of the loss landscape:}
    For normalized FFNs, we prove that the loss landscape becomes exponentially smoother with respect to the number of normalization layers. Consequently, both the \emph{iteration complexity} (lower bound) and \emph{convergence rate} (upper bound) of gradient-based optimization improve exponentially.
    
    \item \textit{Generalization bound with local Lipschitz control:}
    We develop a novel generalization bound that incorporates \emph{local} Lipschitz constants, extending beyond prior global Lipschitz-based bounds. Our bound explains why a network can generalize well even when it is not globally Lipschitz continuous. When combined with our capacity analysis, this bound formally justifies the generalization advantage conferred by multiple normalization layers.
\end{itemize}

Overall, this work goes beyond simplified or shallow settings in prior studies to provide a rigorous theoretical understanding of normalization layers in modern deep architectures. While our analyses focus on feedforward networks, the revealed principles are general and can naturally extend to other architectures such as convolutional and transformer-based networks.

\emph{Organization:} Section \ref{sec-Related-work} provides further related work and some necessary notations. The main analyses on existing normalizers appears in Section~\ref{sec-normalizers-Lipschitz}. In Section~\ref{sec-benefit-controlling-capacity}, we show some extreme behaviors of unnormalized DNNs and discuss the exponential benefit of normalizers to control the capacity of a DNN through its Lipschitz constant. Section~\ref{sec-benefit-generalization-training} discusses the benefits of normalizers to both the generalization ability and training of DNNs, while Section~\ref{sec-Conclusion} contains some concluding remarks.

\section{Related work and Notations} \label{sec-Related-work}

\subsection{Further related work}

The ability of normalization methods \citep{ioffe2015BNorm,ioffe2017batchRe,bjorck2018understandingBN,awais2020revisitingBN} to mitigate the adverse effects of highly variable inputs is one of their most prominent advantages. This property contributes to both stable optimization and improved generalization. To formally understand these benefits, a series of studies have examined normalization from several theoretical perspectives, summarized below.

\textit{Optimization Aspect:} Normalization has been extensively studied from the standpoint of optimization dynamics. \cite{santurkar2018BNorm} demonstrated that batch normalization  reduces the Lipschitz constant of the loss function, thereby facilitating optimization. This property enables faster convergence and supports the use of larger learning rates \citep{luo2019UnderstandingBN,bjorck2018understandingBN,zhang2019fixup,li20BN,arora19BN}. \cite{karakida2019normalization} further showed that applying $\BN$ to the final layer can alleviate pathological sharpness, an undesirable property from the optimization viewpoint, although normalization in earlier layers does not yield the same effect. In the linear regression setting, \citet{cai2019quantitative} proved that training with $\BN$ leads to faster convergence. Similarly, \citet{kohler2019convergenceBN} established that $\BN$ can exponentially accelerate convergence, though their result holds only for a perceptron model with Gaussian inputs. They also observed that normalization weakens inter-layer dependencies, effectively simplifying the curvature structure of the network. 

\citet{lyu22understandingNorm} analyzed gradient descent with weight decay (GD+WD) for minimizing scale-invariant losses, common in normalized DNNs, and showed that near a local minimum, GD+WD drives the parameter vector to oscillate along the manifold of minimizers while approximately following a sharpness-reduction flow. This suggests that normalization implicitly encourages gradient descent to reduce loss sharpness. However, their results do not address the benefit of normalization on iteration complexity or convergence rates, whereas the analyses in \citep{karakida2019normalization,cai2019quantitative,kohler2019convergenceBN} are restricted to simplified models that remain far from modern deep architectures. Very recently, \cite{cisneros25WN} analyzed the benefits of \textit{Weight Normalization} (WN) for training DNNs. However, their results do not apply to other normalization methods (for the inputs) studied in this work.

\textit{Approximation Aspect:} From the approximation perspective, several works have revealed the expressive benefits of normalization.
\citet{frankle21BN} empirically showed that the affine parameters of $\BN$ layers possess substantial expressive power, training only the $\BN$ layers can yield surprisingly high performance. \citet{mueller2023normalization} found that applying sharpness-aware minimization (SAM) solely to normalization layers performs comparably to applying SAM across the entire network, suggesting that normalization layers play a critical role in shaping model expressivity. Strengthening these empirical findings, \citet{burkholzbatch24} theoretically established that, given sufficiently wide and deep architectures, training normalization layers alone is sufficient to represent a broad class of functions.

\textit{Generalization Aspect:} Normalization is also empirically observed to enhance the generalization ability of DNNs \citep{luo2019UnderstandingBN,lyu22understandingNorm}. Despite this, theoretical explanations remain limited. \citet{lyu22understandingNorm} derived a generalization bound for scale-invariant losses using the notion of spherical sharpness, which formally indicates that normalization can reduce sharpness and thereby improve generalization. This provides a partial theoretical explanation for the role of normalization. However, the bound in \citep{lyu22understandingNorm} applies to idealized normalizers and does not capture practical implementations that incorporate smooth constants to prevent division by near-zero variance, leading to scale-varying loss \citep{arora19BN}.  Moreover, it does not account for the cumulative effect of multiple normalization layers within a deep architecture. In contrast, our work establishes a new generalization bound that explicitly characterizes the benefit of multiple normalizers in deep networks.

\cite{cisneros25WN} showed the benefits of WN for DNNs, by developing a generalization bound which bases on the Rademacher complexity of WN-based networks. Similar with those by \cite{golowich2020RC}, the bound in \citep{cisneros25WN} depends polynomially on the network depth. Despite working well with WN-based networks, their bound cannot apply to other types of normalization methods, such as $\BN, \LN, \GN, \IN$. This is a major limitation, making their bound to be entirely different from ours in this work.

\textit{Capacity Control Aspect:} Another influential line of research interprets normalization as a mechanism for capacity control. \citet{santurkar2018BNorm} proved that a single $\BN$ operation can reduce the Lipschitz constant of the loss, implying a regularization effect. However, extending their analysis to networks with multiple normalization layers is nontrivial. \citet{luo2019UnderstandingBN} investigated the regularization behavior of $\BN$ in a single-layer perceptron, and \citet{lyu22understandingNorm} showed that GD+WD implicitly leads to flatter minima. These studies collectively highlight the regularizing role of normalization but remain confined to shallow or idealized architectures. To the best of our knowledge, no existing work provides an explicit characterization of regularization or capacity control in \textbf{deep} normalized networks. In contrast, our analysis demonstrates that popular normalization schemes yield \textit{exponential} improvements in capacity control, convergence rate,  iteration complexity and generalization error, thereby offering a unified theoretical account of their empirical effectiveness.

\subsection{Notations}

A bold character (e.g., $\vx$) often denotes a vector, while a bold big symbol (e.g., $\mD$) often denotes a matrix or set. Denote $\| \cdot \|$ as the $\ell_2$ norm. Given $ \vx =(x_1, ..., x_n) \in \R^n$ and a constant $a$, we denote $\vx -a = (x_1 -a, ..., x_n -a)$. Later we often write $\mA\vx$ which is the multiplication of a matrix $\mA$ with  vector $\vx$ organized in a suitable column. We denote $\epsilon$ as the (small) smoothing constant for some normalizers. $\vx_S$ denotes a subset of $\vx$, corresponding to an index set  $\mS \subseteq [n]$. $| \mS |$ denotes the size/cardinality of $\mS$.

A function $\vy: (\gX, d_x) \rightarrow (\gY, d_y)$ is said to be \textit{$L$-Lipschitz continuous} if $ d_y(\vy(\vx), \vy(\vx'))  \le L d_x (\vx, \vx')$ for any $\vx, \vx' \in \gX$, where $d_x$ is a metric on $\gX$, $d_y$ is a metric on $\gY$, and $L \ge 0$ is the Lipschitz constant.
For simplicity we will consider both metrics to be $\ell_2$ distance, i.e., $d_x(\vx_1, \vx_2) = \| \vx_1 - \vx_2 \|$ and $d_x(\vy_1, \vy_2) = \| \vy_1 - \vy_2 \|$. 
Denote  $\| \vy \|_{Lip}$ as the Lipschitz constant  with respect to its input. For a function $\vh(\mW, \vx)$ depending on two variables, sometimes we write $\| \vh, \mW\|_{Lip}$ to denote the  Lipschitz constant of $\vh$ w.r.t. $\mW$ for a fixed $\vx$. Similarly, sometimes we can use $\| \vh, \vx\|_{Lip}$, for a fixed $\mW$. 
For a linear operator $\mA\vx$, its Lipschitz constant is  $\| \mA\vx \|_{Lip} = \| \mA\| $ which is the spectral norm of matrix $\mA$. 

$O(g(n))$ refers to a function which does not grow faster than a constant multiple of $g(n)$ for large $n$. Similarly $\Omega(g(n))$ refers to a function which grows at least as fast as a constant multiple of $g(n)$ for large $n$.

For a given vector $\vv = (v_1,...,v_n)$ and an integer $m$, denote 
\begin{equation}
\nonumber
\vv^{(m)} = \begin{cases}
(v_1,...,v_m)  & \text{ if } m \le n, \\
\left(\vv, \textbf{0}\right) & \text{ otherwise }
\end{cases}
\end{equation} 
as the result of truncation/zero-padding to $\vv$ to make a vector of size $m$. Also $\vv^{(m_1, m_2)} = \left( \vv^{(m_1)}\right) ^{(m_2)}$ is the result of truncation/zero-padding to $\vv^{(m_1)}$  to make a vector of size $m_2$. Sometimes we denote $\vv^{(m_1,..., m_K)}$ as the result of a series of truncation/zero-padding to $\vv$ with respect to the sequence $m_1,..., m_K$. 


\section{Lipschitz continuity of normalizers} \label{sec-normalizers-Lipschitz}

We first analyze three popular normalization methods: $\BN, \LN$, and $\GN$. Some estimates of their Lipschitz constants are presented.

\subsection{Batch normalization}

$\BN$ \citep{ioffe2015BNorm} normalizes individual inputs at a layer according to their distributions. It means that two different inputs will be normalized independently. 

Let $\{x^{(1)},..., x^{(m)}\}$ be the samples of a signal/variable $x$ with population (either true or estimated) mean $\mu_x$ and variance $\sigma_x^2$. $\BN$ will normalize each sample $x^{(i)}$ as:
\begin{eqnarray}
\label{eq-BN-01}
\BN(x^{(i)},  \epsilon) &=& \frac{1}{\sqrt{\sigma_x^2 + \epsilon}}  (x^{(i)} - \mu_x)
\end{eqnarray}
where the shift and scale parameters are omitted for simplicity. Note that $\partial \BN(x,  \epsilon) / \partial x = {1}/{\sqrt{\sigma_x^2 + \epsilon}}$, since $\mu_x$ and  $\sigma_x^2$ represent the population distribution of signal $x$. This Jacobian is large only when the variance of $x$ is small. It suggests that the magnitude of $\|\BN \|_{Lip}$  depends on the nature of the distribution of signal $x$ only, but not on  other signals in the same layer. The following lemma provides an estimate.

\begin{lemma}\label{lem-BN-Jacobian-norm}
Given $\epsilon>0$, let $\vx = (x_1,..., x_n)$ be an input and $\BN(\vx, \epsilon)$ be the normalization of $\vx$, where each input $x_k$ with  population variance $\sigma_{k}^2 $ is normalized as (\ref{eq-BN-01}). Then 
$ \| \BN \|_{Lip} =  \|{1}/{\boldsymbol{\sigma}} \|$, where ${1}/{\boldsymbol{\sigma}} = \left( (\sigma_1^2 + \epsilon)^{-0.5}, ..., (\sigma_n^2 + \epsilon)^{-0.5} \right)$.
\end{lemma}

This observation, whose proof appears in Appendix~\ref{app-BN}, suggests that  $\|\BN \|_{Lip}$ will be small when all the variances of the inputs are large. An increase in  variances will decrease  $\|\BN \|_{Lip}$. Moreover, a higher-varying input can be penalized stronger. This fact provides us an idea about capacity control of $\BN$.\footnote{In practice, the population variance $\sigma^2_{k}$ for each input $x_k$ is not known and often estimated from the training data \citep{ioffe2015BNorm}. Such an estimate and hence the effectiveness of $\BN$ are  largely affected by batch size. However, the  variances estimated from the training phase are sufficient for analysis on capacity control and generalization of a trained network. Therefore our main messages are essentially the same, although each $\sigma^2_{k}$ is either the truth or estimate.}

\subsection{Layer normalization}

$\LN$ \citep{ba2016LNorm} is an operator that can take an input $\vx \in  \R^n$ and output $\LN(\vx, \epsilon)$ where 
\begin{eqnarray}
\label{eq-LN-01}
\LN(\vx, \epsilon) &=& \frac{1}{\sqrt{\sigma^2_n + \epsilon}}  (\vx - \mu) \\
\label{eq-LN-02}
\mu = \frac{1}{n} \sum_{i=1}^{n} x_i, & &
\sigma^2_n = \frac{1}{n} \sum_{i=1}^{n} (x_i - \mu)^2
\end{eqnarray}
Note that the variance in $\LN$ is different in nature from that in $\BN$. We use directly the input vector $\vx$ to compute variance for $\LN$, making correlations between the outputs of a layer. Therefore, the behavior of $\LN$ may be significantly different from that of $\BN$. The following lemma shows some properties, whose proof appears in Appendix \ref{app-LN}.

\begin{lemma}\label{lem-LN-Jacobian-norm}
Consider any $\vx \in \R^n$ and $\sigma_n$ defined as (\ref{eq-LN-02}). For any $\epsilon>0$, we have $\left\| \frac{\partial \LN(\vx, \epsilon)}{\partial \vx} \right\| \le \frac{1}{\sqrt{\sigma^2_n + \epsilon}}$. Denoting $\sigma  = \min_{\vx \in \R^n} \sqrt{\sigma_n^2 + \epsilon}$, we further have  $\| \LN \|_{Lip}  \le 1 / \sigma$.
\end{lemma}

Lemma \ref{lem-LN-Jacobian-norm} shows that the sensitivity of $\LN$ is entirely governed by its variance normalization term. Geometrically, $\LN$ first projects the input onto the mean-zero subspace and then rescales it by the inverse empirical standard deviation. The projection step is non-expansive (spectral norm equal to 1), so any potential amplification arises solely from the scaling factor $(\sigma_n^2+\epsilon)^{-0.5}$. The lemma therefore formalizes a clean operator-theoretic statement: $\LN$ cannot expand perturbations by more than $(\sigma_n^2+\epsilon)^{-0.5}$. In particular, the stabilizer $\epsilon$ ensures global Lipschitz continuity, preventing gradient blow-up when the input variance approaches zero. 

Compared with $\BN$, this behavior is structurally simpler and more localized. $\LN$ operates independently on each sample, so its Jacobian and Lipschitz constant depend only on that sample's feature variance. In contrast, $\BN$ empirically estimates the mean and variance of each input from a batch and hence $\BN$ normalizes across the batch, making its Jacobian couple different examples and rendering its sensitivity dependent on batch statistics and inter-sample correlations. Consequently, $\LN$ admits a deterministic per-sample Lipschitz bound, while $\BN$'s effective smoothness is inherently batch-dependent. 

\subsection{Group normalization}

$\GN$ \citep{wu2020GNorm} is a slight modification of $\LN$. Instead of normalizing all the input signals at a layer by the same statistics, $\GN$ takes a group of those signals and then normalizes them. This normalizer makes effects more local. The  number of groups is a hyperparameter. 

Let $\vx = (x_1,..., x_n)$ be the input at a layer, where $n \ge 2$.  For an index subset $\mS \subseteq \{1,...,n\}$, $\GN$ takes input $\vx_S$ from $\vx$ and outputs $\GN(\vx_S, \epsilon)$, where:
\begin{eqnarray}
\label{eq-GN-01}
\GN(\vx_S, \epsilon) &=& \frac{1}{\sqrt{\sigma_S^2 + \epsilon}}  (\vx_S - \mu_S), \\
\mu_S = \frac{1}{|\mS|} \sum_{i \in \mS} x_i, & &
\sigma_S^2 = \frac{1}{|\mS|} \sum_{i \in \mS} (x_i - \mu_S)^2
\end{eqnarray}

$\GN$ will become $\LN$ when $n = |\mS|$. The following lemma, which is proven in Appendix \ref{app-GN}, shows some properties. 

\begin{lemma}\label{lem-GN-Jacobian-norm}
Let $\bigcup_{k=1}^C \mS_k$ be a decomposition of $\{1,...,n\}$ into $C$ disjoint subsets, each of size at least 2. Assume that those subsets are used by $\GN$ to normalize an input $\vx \in  \R^n$. Denote $\sigma_k^2(\vx)$ as the variance  to compute (\ref{eq-GN-01}) for  group $\mS_k$, and $\sigma_k^2 = \min_{\vx \in \R^n} \sigma_k^2(\vx)$.  We have $\| \GN \|_{Lip}  \le \sum_{k=1}^C (\sigma_{k}^2 + \epsilon)^{-0.5}$.

\end{lemma}

\section{Capacity control of popular normalization methods} \label{sec-benefit-controlling-capacity}

Normalization methods are widely recognized for their ability to stabilize training and improve the predictive accuracy of DNNs. While previous studies have analyzed their benefits in simplified or shallow settings, the cumulative effect of multiple normalization layers in deep architectures remains insufficiently understood. In this section, we show that, for a feedforward network (FFN), the model family can simultaneously contain members with excessively large or abnormally weak capacities, resulting in an extremely complex loss landscape with uncountably many sharp and flat regions that lead to gradient explosion or vanishing. Moreover, we show that multiple normalization layers can impose an exponential penalty on the  network capacity.

\subsection{The Lipschitz  constant of a DNN} \label{subsec-Lipshitz-constant-DNN}

When training a neural network, a learning algorithm often searches for a specific member in a family $\gH$ of models (hypotheses). In practice, the weight matrices of a neural network are often bounded above. Therefore it is natural to consider the following family.

\begin{definition}[DNN] \label{def-FCN}
Let fixed activation functions $(g_1, \dots, g_K)$, whose Lipschitz constants are at most 1, and the bounds $(s_1, \dots, s_K)$. 
Let $\vh_{i} = g_i(\mW_i \vh_{i-1})$, with $\vh_{0} = \vx$, be the feedforward network associated with weight matrices $(\mW_1, \dots, \mW_i)$. Let $\gH = \{\vh_K(\vx): \vx \in \R^n,  \| \mW_i \| \le s_i, \forall i \le K\} $ be the family of neural nets with $K$ layers. 
\end{definition}

Note that most popular activation functions (e.g., ReLU, Leaky ReLU, Tanh, Sigmoid,  Softmax) have Lipschitz constants being at most 1. The upper bound on the weight norms is practically natural and is often used in prior theories \citep{bartlett2017SpectralMarginDNN,neyshabur2018SpectralMarginDNN,arora2021dropout}. 

Next we consider the Lipschitz constant of a DNN. Note that such a constant tells the complexity of a function. A larger Lipschitz constant implies the function of interest may change faster in the areas  around some inputs. It also suggests that the function may be more complex. Therefore, the Lipschitz constant can tell the complexity of a DNN (and its family). A simple observation reveals the following  property, whose proof is in Appendix~\ref{app-DNN-proofs-normalization}.

\begin{lemma}[Upper bound] \label{lem-Lipschitz-FCN}
$\left\| \vh, \vx \right\|_{Lip}  \le  P_w(\vh)$, for all $ \vh \in \gH$  defined in Definition~\ref{def-FCN}, where $P_w(\vh) = \prod_{k=1}^{K} \|\mW_{i}\|$. 
\end{lemma}

This lemma says that the Lipschitz constant of any member in family $\gH$ is at most $\prod_{k=1}^{K} s_k$. Such an upper bound holds true for all members of $\gH$. Note that this upper bound can be loose for a specific member of interest. One can wonder whether or not this upper bound is too pessimistic. To this end, we reveal the following property  whose proof appears in Appendix~\ref{app-DNN-proofs-no-norm}.

\begin{theorem}[Lower bound] \label{thm-Lipschitz-FCN-lower-bound}
Consider the family $\gH$  defined in Definition~\ref{def-FCN} with ReLU activation functions. For any $a_i \in [0, s_i], i \le K$, there exists a member $\vh^* \in \gH$ satisfying  $\left\| \vh^*, \vx \right\|_{Lip}  =  \prod_{k=1}^{K} a_k$.
\end{theorem}

\begin{figure}[t]
\centering
\includegraphics[width=0.4\textwidth]{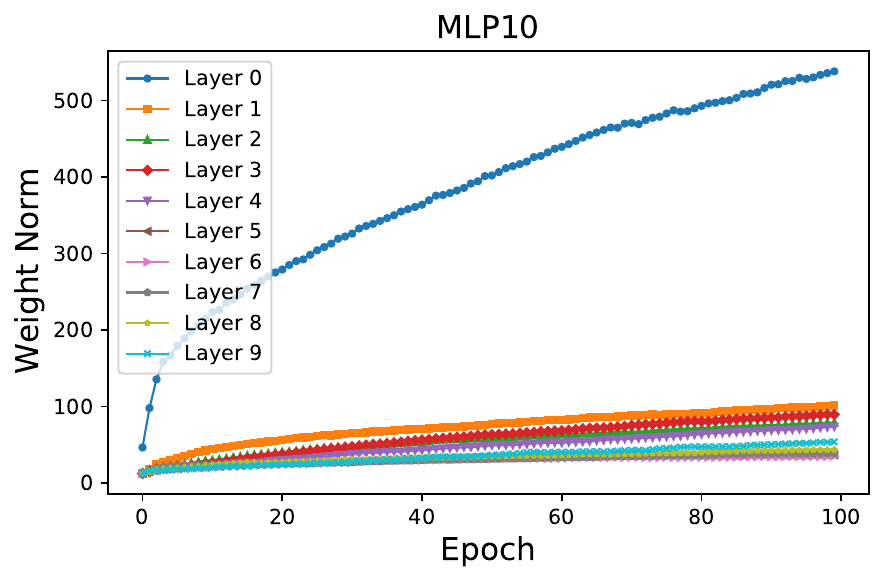}
\hspace{5pt}
\includegraphics[width=0.26\textwidth]{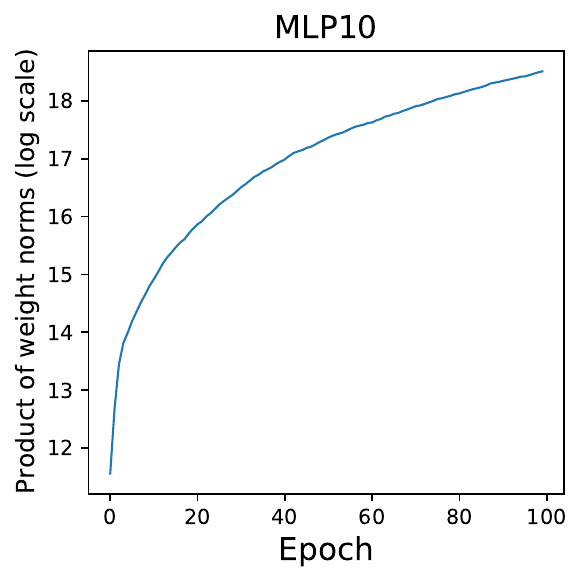}
\hspace{5pt}
\includegraphics[width=0.26\textwidth]{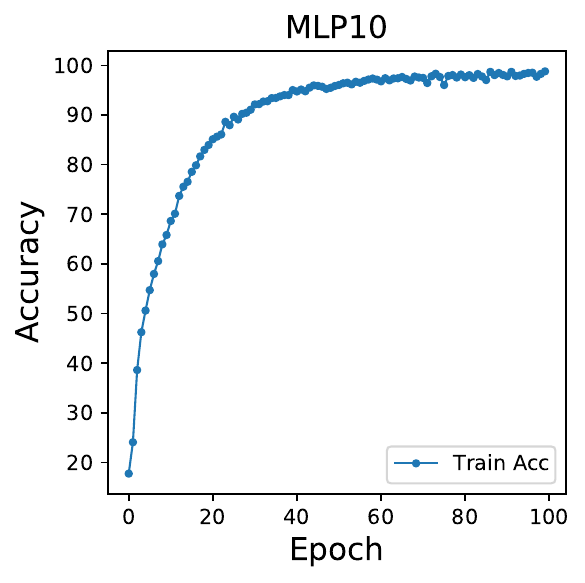}
\\
\includegraphics[width=0.4\textwidth]{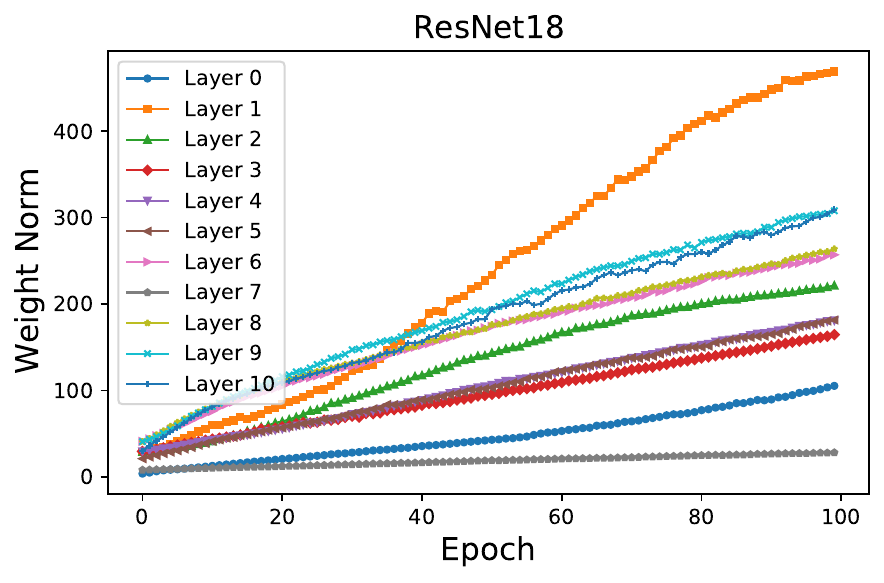}
\hspace{5pt}
\includegraphics[width=0.26\textwidth]{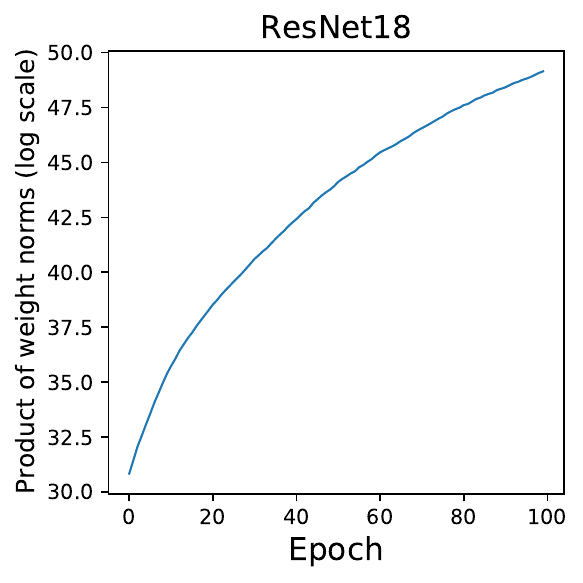}
\hspace{5pt}
\includegraphics[width=0.26\textwidth]{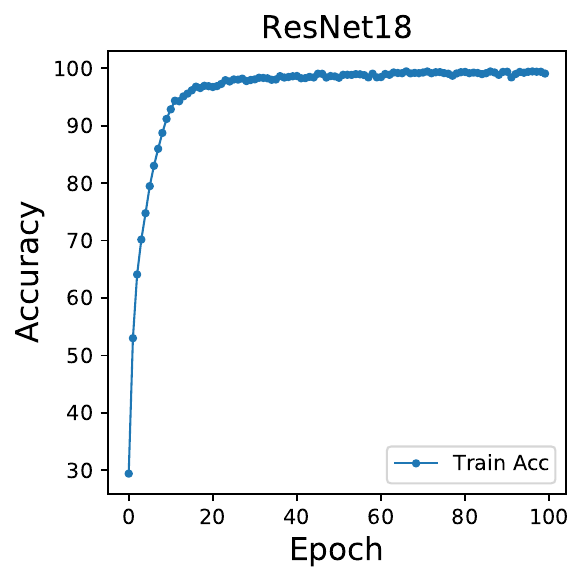}
\caption{The dynamics along the training process. The leftmost subfigures present  the weight norm at each layer,  the middle subfigures report the product of all weight norms, while the rightmost subfigures report the training accuracy. CIFAR10 dataset is used to train ResNet18 and a ReLU network with 10 layers. For ResNet18, only the dynamics of the first 11 weight matrices are presented for clarity. Detailed settings can be found in Appendix~\ref{app-empirical-evaluations}.} 
\label{fig:MLP-weight-norms}
\end{figure}

\begin{remark}\label{remark-01-input-Lipschitz}
These  results lead to several  implications:
\end{remark}
 \begin{itemize}
 \item \textit{Uncountably many members with exponentially large Lipschitz constants:}
 This phenomenon becomes evident when $s_i > 1$ for all $i$. In such cases, the parameters $a_i$ can take uncountably many values within the interval $(1, s_i]$, resulting in an uncountable set of models whose Lipschitz constants are of order $\prod_{k=1}^{K} a_k$. The product grows exponentially with respect to the network depth $K$, implying that even moderate increases in depth can lead to extremely large Lipschitz values.
 
 From a learning perspective, this exponential growth can introduce a major challenge: \textbf{overfitting}. A large Lipschitz constant indicates that the model is highly sensitive to small input perturbations, which can make it fit noise rather than structure in the training data. 
 
 Empirical evidence supporting this observation is illustrated in Figure~\ref{fig:MLP-weight-norms}. When training ResNet18 and a deep ReLU network, the norms of weight matrices at all layers often become significantly larger than one as training progresses. This behavior suggests that the corresponding norm bounds $s_i$ must also be large, and hence the overall bound $\prod_{k=1}^{K} s_k$ can become astronomically high for modern DNNs (e.g., $10^{50}$ for ResNet18). Figure~\ref{fig:preBN-variance-Norm-EfficientNet} in Appendix~\ref{app-empirical-evaluations} reports the same behavior for EfficientNet, which has extreme weight norms. Such exponential growth highlights the necessity of regularization  techniques to control the Lipschitz behavior in practice.
 
 \item \textit{Uncountably many members with exponentially small Lipschitz constants:}
 Conversely, for numbers $a_i < 1$, there exists models whose Lipschitz constants decay exponentially with depth. While these networks are mathematically valid members of the family $\gH$, they tend to exhibit limited expressive capacity. A very small Lipschitz constant implies that the network output changes only marginally with respect to the input, resulting in overly smooth or nearly constant mappings. Such models typically suffer from \textbf{underfitting}, as they cannot capture the complexity of real-world data. 
 
 \item \textit{Tightness of Lemma~\ref{lem-Lipschitz-FCN}:}
 The above observations also confirm that the Lipschitz bound provided in Lemma~\ref{lem-Lipschitz-FCN} is tight for the hypothesis family $\gH$. That is, without imposing additional structural or statistical assumptions, it is impossible to derive a substantially tighter upper bound on the Lipschitz constant. The lemma therefore captures the inherent variability within $\gH$---ranging from extremely contractive to highly expansive mappings---and highlights the importance of further assumptions if one aims to obtain sharper control over model stability and generalization.
 \end{itemize}

Although  Theorem \ref{thm-Lipschitz-FCN-lower-bound} reveals these properties for ReLU networks, we believe that these properties can appear in many other types of DNNs.

\subsection{Smoothening DNNs by normalization} \label{subsec-smoothen}

Next we  want to see the role and benefits of a normalizer. To this end, we consider the following normalized DNNs.

\begin{definition}[Normalized DNN] \label{def-normalized-FCN}
Let fixed activation functions $(g_1, \dots, g_K)$, whose Lipschitz constants are at most 1, and the bounds $(s_1, \dots, s_K)$. 
Let $\vh_{i} = g_i(\NO_i(\mW_i \vh_{i-1}))$, with $\vh_{0} = \vx$, be the neural network associated with weight matrices $(\mW_1, \dots, \mW_i)$, where $\NO_i$ denotes a normalization operator at layer $i$. Let $\gH_{no} = \{\vh_K(\vx): \vx \in \R^n,  \| \mW_i \| \le s_i, \forall i \le K\} $ be the family of normalized DNNs with $K$ layers.  
\end{definition}

The first property is revealed in the following result, whose proof appears in Appendix~\ref{app-DNN-proofs-normalization}.

\begin{theorem} \label{lem-Lipschitz-FCN-normalize}
$\left\| \vh_{no}, \vx \right\|_{Lip}  \le  P_w(\vh_{no}) \prod_{k=1}^{K} \| \NO_k\|_{Lip}$, for all $ \vh_{no} \in \gH_{no}$  defined in Definition~\ref{def-normalized-FCN}, where where $P_w(\vh_{no}) = \prod_{k=1}^{K} \|\mW_{i}\|$. 
\end{theorem}

Consider the case of $\BN$.  \cite{ioffe2015BNorm} observed that the distribution of input at a layer may vary significantly and hence cause some difficulties for training. The high varying inputs imply a high variance of the input distribution, i.e., a large $\sigma_{x}$ in (\ref{eq-BN-01}). $\BN$ was proposed to resolve the issue of high variances, possibly helping training more easily. The following results   provide a novel perspective, which comes from combining Theorem~\ref{lem-Lipschitz-FCN-normalize} and Lemma~\ref{lem-BN-Jacobian-norm}.

\begin{corollary}[DNN+BN] \label{cor-FCN-Lipschitz-constant-BN}
Consider  any $\vh_{no} \in \gH_{no}$ in Definition~\ref{def-normalized-FCN} with  operator $\NO_k(\cdot) \equiv \BN(\cdot, \epsilon)$  at any layer $k \le K$. Let $\boldsymbol{\sigma}_k$ be the input variances at layer $k$ as defined in Lemma~\ref{lem-BN-Jacobian-norm}. Then:

1. $ \left\| \vh_{no},\vx \right\|_{Lip}   \le P_w(\vh_{no}) \prod_{k=1}^{K} \|1/{\boldsymbol{\sigma}_k}\| $. 

2. $ \left\| \vh_{no}, \vx \right\|_{Lip}   \le \sigma^{-K} P_w(\vh_{no})$, where $\sigma = \min_k \| 1/\boldsymbol{\sigma}_k \|^{-1}$.
\end{corollary}

This corollary indicates that $\BN$ can have a strong penalty on family $\gH_{no}$. As the input variance increases, $\| \vh_{no}, \vx \|_{Lip} $ tends to be smaller, meaning that  the penalty  tends to be stronger. Note that  many members of the unnormalized family $\gH$ have a Lipschitz constant of size  $\Omega(\prod_{k=1}^{K} s_k)$, as shown by Theorem~\ref{thm-Lipschitz-FCN-lower-bound}. In contrast, the Lipschitz constant of every member of $\gH_{no}$ is atmost $(\prod_{k=1}^{K} s_k)$ $\prod_{k=1}^{K} \|1/{\boldsymbol{\sigma}_k}\| $. This suggests that $\BN$ can significantly reduce the Lipschitz constant, and hence makes the normalized DNNs smoother than their unnormalized versions.

The following result comes from combining Theorem~\ref{lem-Lipschitz-FCN-normalize} with  Lemma~\ref{lem-LN-Jacobian-norm}. 

\begin{corollary}[DNN+LN] \label{cor-FCN-Lipschitz-constant-LN}
Consider  any $\vh_{no} \in \gH_{no}$ in Definition~\ref{def-normalized-FCN}  with  operator $\NO_k(\cdot) \equiv \LN(\cdot, \epsilon)$  at any layer $k \le K$. Denote $n_k$ as the size of the input and $\sigma_{k}$ as input variance  before $\NO_k$, defined in Lemma~\ref{lem-LN-Jacobian-norm}. Then  
$ \left\| \vh_{no}, \vx \right\|_{Lip}   \le  P_w(\vh_{no}) \prod_{k=1}^{K}  \sigma_{k}^{-1}$.
\end{corollary}

From Corollaries~\ref{cor-FCN-Lipschitz-constant-BN} and~\ref{cor-FCN-Lipschitz-constant-LN}, one can observe that $\BN$ and $\LN$ (and, by extension, $\GN$ and $\IN$) share several theoretical properties. Under the assumption of large input variances, both normalization mechanisms impose strong constraints on the Lipschitz constant of a deep neural network, thereby effectively controlling the overall model complexity and capacity of the family~$\gH$. Importantly, this assumption is not purely theoretical but has been widely observed in practice, as reported in numerous empirical studies \citep{de2020BN,shao2020normalization,ioffe2015BNorm}.

\textit{Empirical evidence of large input variances:}  
To further validate this assumption, we trained a 10-layer ReLU network and a ResNet18 model on the CIFAR-10 dataset and monitored the input variance at each layer throughout training. Figure~\ref{fig:variance-before-normalization} presents the recorded statistics. Owing to the use of He initialization, both models initially exhibit small input variances across all layers, allowing training to progress smoothly during the early mini-batches. However, as training continues, we observe that the input variances at several layers increase rapidly---often exceeding~1 by a large margin. In some layers, the variance continues to grow steadily even after the training error approaches zero, suggesting that activations become increasingly variable as training progresses. The same behavior also appears in EfficientNet-B3 as reported by Figure~\ref{fig:preBN-variance-Norm-EfficientNet} in Appendix~\ref{app-empirical-evaluations}.

\begin{figure}[t]
	\begin{minipage}[b]{0.46\textwidth} 
        \centering
        \includegraphics[width=\textwidth]{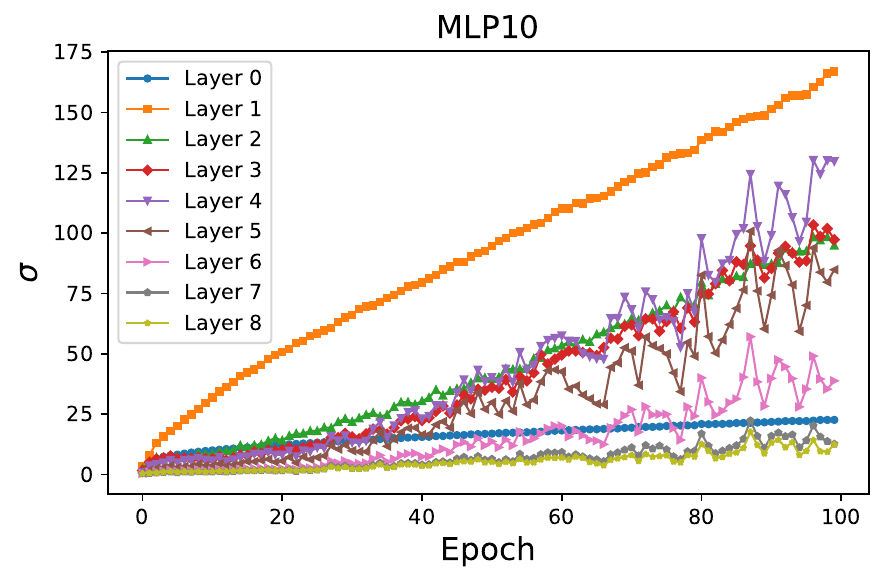} 
    \end{minipage}
    \hfill
    \begin{minipage}[b]{0.3\textwidth} 
            \centering
            \includegraphics[width=\textwidth]{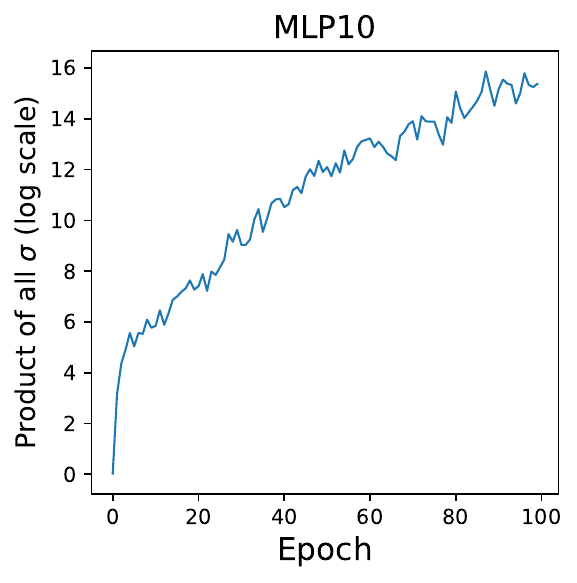} 
        \end{minipage}
        \hfill
     \begin{minipage}[b]{0.46\textwidth} 
             \centering
             \includegraphics[width=\textwidth]{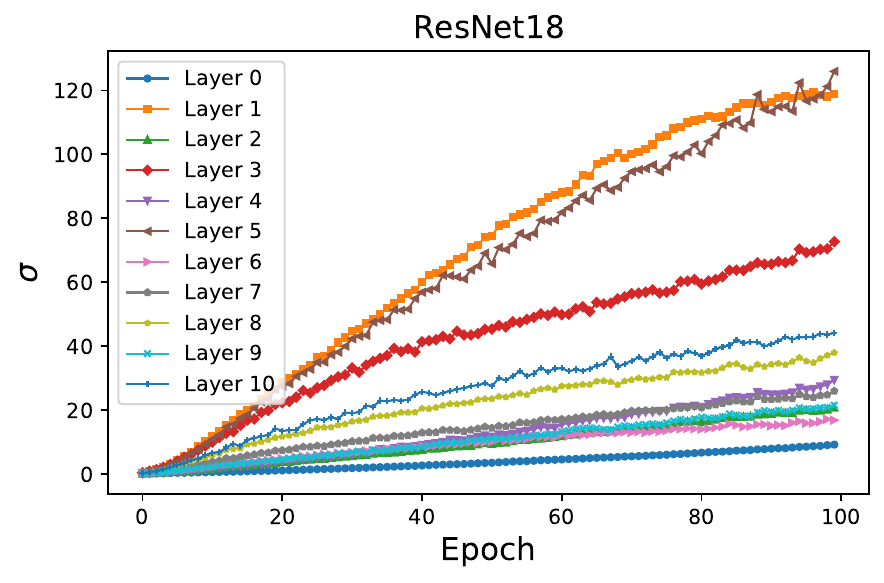}
         \end{minipage}
         \hfill
    \begin{minipage}[b]{0.3\textwidth} 
            \centering
            \includegraphics[width=\textwidth]{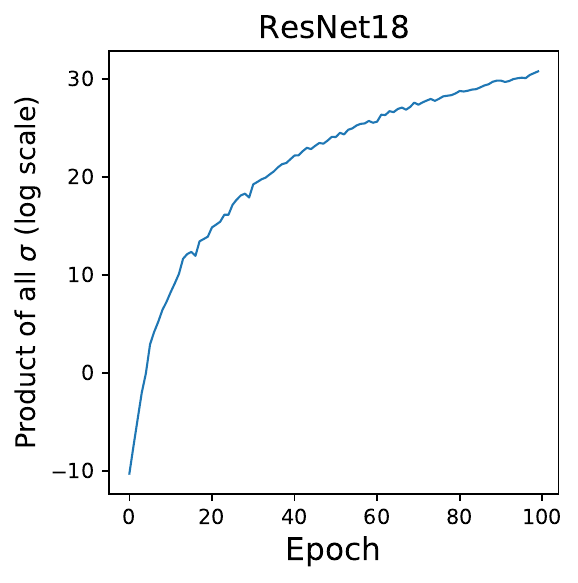}
        \end{minipage}
\caption{Evolution of input variances across layers in a 10-layer ReLU network and ResNet18 trained on the CIFAR-10 dataset. The input variance $(\sigma^2)$ is computed over mini-batches before each activation function (or before each $\BN$ layer in ResNet18). For ResNet18, only the dynamics of the first 11 $\BN$ layers are presented for clarity.  Although both networks are initialized using He initialization, several layers exhibit a rapid increase in input variance during training. Consequently, the cumulative product of layer-wise variances grows approximately exponentially.} 
\label{fig:variance-before-normalization}
\end{figure}

\textbf{Potentially exponential reduction due to normalization:}  
When such large input variances occur, normalization layers can dramatically reduce the network's effective Lipschitz constant. According to Corollary~\ref{cor-FCN-Lipschitz-constant-BN}, $\BN$ yields a reduction factor of order $O\!\left(\prod_{k=1}^{K} \|1/{\boldsymbol{\sigma}_k}\|\right)$, while Corollary~\ref{cor-FCN-Lipschitz-constant-LN} indicates that $\LN$ provides a reduction of order $O\!\left(\prod_{k=1}^{K} 1/\sigma_k\right)$. As shown in Figure~\ref{fig:variance-before-normalization}, the values of $\sigma_k$ are typically much greater than~1 after training. Consequently, the cumulative reduction factors can be extremely large in magnitude---potentially exponential in the number of normalization layers. For instance, in ResNet18 with  20 $\BN$ layers, we empirically observe \textcolor{blue}{$\prod_{k=1}^{K} 1/{\sigma_k} \approx 10^{-25}$}. Such a reduction is remarkably substantial and highlights the profound stabilizing effect of normalization on deep networks.

\subsection{Smoothening the loss landscape} \label{subsec-smoothen-loss-landscape}

We next consider the effects of a normalizer to the loss when  training a DNN. To this end, we need to analyze the behaviors of a DNN w.r.t its weights. By using similar proofs with Theorem~\ref{lem-Lipschitz-FCN-normalize} and Theorem~\ref{thm-Lipschitz-FCN-lower-bound}, one can easily show the following results.

\begin{theorem} \label{thm-Lipschitz-FCN-normalize-weight}
Consider any member $ \vh_{no} \in \gH_{no}$  defined in Definition~\ref{def-normalized-FCN}. Denote $\left\| \vh_{no}, \mW_i \right\|_{Lip}$ as the Lipschitz constant of  $\vh_{no}$ with respect to the weight $\mW_i$, $\left\| \vh_{no}, \vy_i \right\|_{Lip}$ as the Lipschitz constant of  $\vh_{no}$ w.r.t to the total input $\vy_i = \mW_i \vh_{i-1}$ at layer $i \le K$, $b_{i} = \|\mW_{i} \| \cdots \| \mW_K \|$ and $A_i = \sup_{\vx} \| \vh_{i}(\vx) \|$ as the largest norm of the output of layer $i$. Then
\begin{eqnarray}
\label{thm-Lipschitz-FCN-normalize-weight-01}
\left\| \vh_{no}, \mW_i \right\|_{Lip}  \le  A_{i-1} \left\| \vh_{no}, \vy_i \right\|_{Lip} 
\end{eqnarray}
\begin{eqnarray}
\label{thm-Lipschitz-FCN-normalize-weight-02}
\left\| \vh_{no}, \vy_i \right\|_{Lip}  \le \begin{cases}
\| \NO_K\|_{Lip} & \text{if } i = K \\
 b_{i+1} \prod_{k=i}^{K} \| \NO_k\|_{Lip} & \text{otherwise}
\end{cases} 
\end{eqnarray}
\end{theorem}

This theorem reveals that a normalization operator can have a direct control on the behaviors of many parts of a DNN. When the input variances at a layer is large, such a control will be stronger and a normalizer can significantly smoothen those parts of a DNN. Such a control should be crucial for the training process, since the training loss should be smooth in a controllable way. On the other side, an unnormalized network can be very bad as revealed below, whose proof appears in Appendix~\ref{app-DNN-proofs-no-norm}. 

\begin{theorem}[Bad unnormalized DNNs] \label{thm-Lipschitz-FCN-lower-bound-weight}
Consider the family $\gH$  defined in Definition~\ref{def-FCN} with ReLU activations and notations in Theorem~\ref{thm-Lipschitz-FCN-normalize-weight}. For all $i \in [1,K)$ and all  $a_j \in [0, s_j]$ with $j >i$, there exists a member $\vh^* \in \gH$ satisfying: 
\begin{eqnarray}
\label{thm-Lipschitz-FCN-lower-bound-weight-001}
\left\| \vh^*(\vx), \mW_i \right\|_{Lip}  = a_{i+1} \cdots a_K  \left\| \left( \vh_{i-1}(\vx)\right) ^{(n_{i+1},...,n_K)}\right\|, \forall \vx,
\end{eqnarray}
\begin{eqnarray}
\label{thm-Lipschitz-FCN-lower-bound-weight-002}
\left\| \vh^*, \vy_i \right\|_{Lip}  = a_{i+1} \cdots a_K 
\end{eqnarray}
\end{theorem}

This theorem suggests that the Lipschitz constant of an unnormalized DNN w.r.t. the weights  at layer $i$ can be of size $\Omega(s_{i+1} \cdots s_K)$. Such a number can be huge in practice. Indeed, Figure~\ref{fig:MLP-weight-norms} demonstrates how large the weight norms in a FFN could be. This figure suggests that each $s_i$ could be much larger than 1, and hence the product $s_{i+1} \cdots s_K$ could be often exponentially large  in practice. Meanwhile, a suitable use of normalizers at the layers with large input variances can significantly reduce the Lipschitz constant, as pointed out by Theorem~\ref{thm-Lipschitz-FCN-normalize-weight}.

Next we consider the loss  for  training a DNN. For simplicity, we consider the $\ell_1$ loss, which is $f(\vh, \vx) = \| \hat{\vy}(\vx) - \vh(\vx) \|_1$ for measuring the error of a prediction $\vh(\vx)$ at  a given input $\vx$ whose the true label is $\hat{\vy}(\vx)$. Observe that $\| f, \mW_i \|_{Lip} = \left\| \vh, \mW_i \right\|_{Lip}$. Combining this observation with Theorem~\ref{thm-Lipschitz-FCN-normalize-weight} and Theorem~\ref{thm-Lipschitz-FCN-lower-bound-weight} leads to the followings.

\begin{corollary}[Lipschitz constant of the loss] \label{cor-FCN-Lipschitz-constant-loss}
Consider  the family $\gH$  in Definition~\ref{def-FCN} with ReLU activations, family $\gH_{no}$ in Definition~\ref{def-normalized-FCN}, and the $\ell_1$  loss $f$. For any $i \in [1,K)$, any $\vx$  and  $\vh_{no} \in \gH_{no}$,
\begin{equation}
 \label{cor-FCN-Lipschitz-constant-loss-01}
\left\| f(\vh_{no}, \vx), \mW_i \right\|_{Lip} \le O(b_{i+1})  \prod_{k=i}^{K} \| \NO_k\|_{Lip}
\end{equation}
where $b_{i} = \|\mW_{i} \| \cdots \| \mW_K \|$. For all  $a_j \in [0, s_j]$ with $j > i$, there exists a member $\vh^* \in \gH$ satisfying:
\begin{equation}
 \label{cor-FCN-Lipschitz-constant-loss-02}
\left\| f(\vh^*, \vx), \mW_i \right\|_{Lip} = a_{i+1} \cdots a_K  \left\| \left( \vh_{i-1}(\vx)\right) ^{(n_{i+1},...,n_K)}\right\|
\end{equation}
\end{corollary}

This corollary shows that there are some points ($\vh^*$) that make the loss function to be extremely sharp, since the loss around those points can have a local Lipschitz constant of size $s_{i+1} \cdots s_K$ w.r.t. the weights of layer $i$. In fact, it further suggests that \textit{the number of those sharp points can be infinite}. This is problematic for the training phase. The sharper loss could slowdown the training as discussed later. On the other hand, the use of normalization can significantly reduce sharpness. With an appropriate use of a normalizer at the layers with high input variances, we can reduce sharpness at a factor of $\prod_{k=i}^{K} \| \NO_k\|_{Lip}$.  Such a factor can be exponential in the number of normalization operators, as evidenced in Figure~\ref{fig:variance-before-normalization} where the reduction factor is approximately $10^{-25}$. 

To reveal what can happen for a training loss in practice, we consider a ReLU network that outputs a number (instead of a vector) and obtain the following result.

\begin{theorem}[Extreme gradients of the training loss] \label{thm-FCN-gradient-weight}
Consider the model family $\gH_b = \{h(\vx) = g(\mW \vh_{K}(\vx)) : \vh_{K}(\vx) \in \gH, \mW \in \R^{1 \times n_K}\}$, where $\gH$ is defined in Definition~\ref{def-FCN} with ReLU activations and $g$ is a  differentiable activation. Given a differentiable loss function $f$, denote $L(h, \mD) = \frac{1}{m} \sum_{\vx \in \mD} f(h, \vx)$ as the empirical loss of $h$ on a data set $\mD$ of size $m$. For all $i \in [1,K)$,  $a_j \in (0, s_j]$ with $j > i$, and $\forall c >0$, there exists a member $h^* \in \gH_b$ satisfying: 
\begin{eqnarray}
\label{thm-FCN-gradient-weight-01} 
\frac{\partial  L(h^*, \mD)  }{\partial \mW_{it}} = \frac{c \textcolor{red}{a_K \cdots a_{i+1}} }{m}\sum\limits_{\vx \in \mD}  \gI_{it} \frac{\partial   (f \circ g)}{\partial u} \mid_{u = u(\vx)}  (\vh_{i-1}(\vx))^\top 
\end{eqnarray}
for any $t \le \min \{n_{i+1},...,n_K\}$ and $\frac{\partial  L}{\partial \mW_{it}} = 0$ otherwise, where $\mW_{it}$ is the $t$-th row of $\mW_{i}$, $u(\vx) = \mW \vh_{K}(\vx)$, $\circ$ denotes the function composition, and 
$\gI_{it} = \begin{cases}
1 & \text{ if } \mW_{it} \vh_{i-1}(\vx) \ge 0,\\
0 & \text{otherwise}.
\end{cases}$
\end{theorem}

\begin{remark}\label{remark-02-weight-Lipschitz}
The results in these theorem and corollary lead to several  implications:
\end{remark}
\begin{itemize}
\item \textit{Uncountably many gradients with exponentially large magnitude:}  
This case occurs when $s_i > 1, \forall i$. Such configurations are not only theoretically possible but also practically relevant, as suggested by Figure~\ref{fig:MLP-weight-norms}, which shows that weight matrices in both FFNs and modern DNNs often develop norms much larger than~1 during training.  

In these cases, each parameter $a_i$ can take uncountably many values within the interval $(1, s_i]$. According to Equation~(\ref{thm-FCN-gradient-weight-01}), the gradient of the loss with respect to the weights of the $i$-th layer can have a size of  $\Omega(a_K \cdots a_{i+1})$, which grows exponentially with the network depth~$K$. Consequently, there exist uncountably many weight configurations that yield gradients of exponentially large size. This observation highlights the potential emergence of the \textbf{gradient explosion} problem, where excessively large gradients cause instability in optimization or numerical overflow. It also suggests that the loss landscape can be extremely sharp at uncountably many points. Those findings emphasize the importance of regularization.

\item \textit{Uncountably many gradients with exponentially small magnitude:}  
In contrast, when $a_i < 1, \forall i$, the product $a_K \cdots a_{i+1}$ decays exponentially with depth. This situation leads to gradients whose magnitudes diminish rapidly as they propagate backward through layers, corresponding to the well-known \textbf{gradient vanishing} phenomenon.  

Because each $a_i$ can take uncountably many values within $(0,1)$, there exist uncountably many configurations that yield nearly zero gradients. This implies that non-negligible portions of the loss landscape can become extremely flat, causing optimizers to make little or no progress. This issue fundamentally limits the trainability of deep architectures without special design choices such as residual connections.

\item \textit{Gradient control via normalization:}  
The result in Equation~(\ref{cor-FCN-Lipschitz-constant-loss-01}) further indicates that normalization mechanisms can effectively regulate the magnitude of gradients throughout the training process. By constraining or adaptively rescaling the intermediate activations, normalizers can keep the gradients within a stable range. This can prevent  the exponential growth of gradients, allowing more reliable optimization and smoother convergence. From a theoretical viewpoint, normalization thus serves as an implicit form of Lipschitz control that bounds the sensitivity of the loss with respect to weight perturbations, helping to ensure stable gradient flow across deep networks.
\end{itemize}

\begin{remark}\label{remark-loss-Lipschitz}
In summary, we have uncovered several intriguing properties concerning the gradients and Lipschitz continuity of deep networks and their associated losses. Theorem~\ref{thm-FCN-gradient-weight} indicates that, in an unnormalized network, the training loss is $O(s_K \cdots s_{i+1})$-Lipschitz continuous with respect to the weight matrix $\mW_i$, and $O(s_K \cdots s_2)$-Lipschitz continuous with respect to the first layer's weight matrix ($\mW_1$). In practice, these Lipschitz constants tend to be extremely large, as illustrated in Figure~\ref{fig:MLP-weight-norms}, implying that the loss landscape can be highly sensitive and contain sharp regions, while also exhibiting flat plateaus elsewhere.

In contrast, normalized networks exhibit much smoother behavior. According to Theorem~\ref{thm-Lipschitz-FCN-normalize-weight}, the loss becomes $O(s_K \cdots s_{i+1} \prod_{k=i}^{K} \| \NO_k \|_{Lip})$-Lipschitz continuous with respect to $\mW_i$. Importantly, the factor $\prod_{k=i}^{K} \| \NO_k \|_{Lip}$ is inversely proportional to the input variances and can thus be exponentially small in practice, as supported by the empirical results in Figure~\ref{fig:variance-before-normalization}. A similar contrast holds for the Lipschitz constant with respect to the input: unnormalized DNNs can exhibit extremely high sensitivity, whereas normalized architectures maintain bounded smoothness. These distinctions have profound implications for both the training dynamics and the generalization behavior of deep networks.
\end{remark}

\section{Benefits in generalization and training} \label{sec-benefit-generalization-training}

We have enough tools to show the  significant benefits of normalization methods for the training phase and generalization ability of a trained DNN. 

\subsection{Benefits on training}

Many existing works \citep{ioffe2015BNorm,ioffe2017batchRe,bjorck2018understandingBN,awais2020revisitingBN} empirically show that normalization methods can significantly improve stability for training DNNs. This ability of normalizers enables faster training with larger learning rates \citep{luo2019UnderstandingBN,bjorck2018understandingBN,zhang2019fixup}.  \cite{santurkar2018BNorm} shows both empirically and theoretically that $\BN$ can reduce the Lipschitz constant of the loss and make gradients more stable and predictable. This seems to be the first work to show the capacity control of $\BN$. However, their analysis only focuses on one $\BN$ layer, and it is non-trivial to extend their analysis to a DNN with many $\BN$ operators or other normalizers. 

In the previous section, we reveal that the Lipschitz constant of a DNN can be reduced exponentially as more normalization operators are appropriately used. This reduction directly translates to the training loss. As summarized in Remark~\ref{remark-loss-Lipschitz}, the training loss is $O(s_K \cdots s_2)$-Lipschitz continuous with respect to the first layer's weight matrix ($\mW_1$), for an unnormalized DNN. Meanwhile, a normalizer can reduce the Lipschitz constant by a factor which can be exponential in the number of normalization layers. Such a reduction leads to a profound benefit on both iteration complexity (lower bound) and convergence rate of a training algorithm. To see this, we summarize  some existing results in the optimization literature.

\textit{Iteration complexity:} Consider the problem of minimizing a function $\ell(\vw)$. For  convex problems, any first-order method requires \textit{at least} $\Omega( \| \ell, \vw \|_{Lip}^2/ \alpha^2)$ iterations to find an $\alpha$-approximate solution, while $\Omega( \| \ell, \vw \|_{Lip}^2/ \alpha)$ iterations are required for strongly convex problems  \citep{bubeck2015convex}. On the other hand, for nonsmooth nonconvex problems, \cite{zhang2020OPTcomplexity} show that generalized gradient based methods need at least $\Omega( \| \ell, \vw \|_{Lip}/ \alpha)$ iterations to find an $\alpha$-stationary point. 
Combining those results with Corollary \ref{cor-FCN-Lipschitz-constant-loss}, the following result readily follows.

\begin{corollary} 
Consider the training loss $L$, defined in Theorem \ref{thm-FCN-gradient-weight}, for an unnormalized DNN. Any generalized gradient based method needs at least $\Omega( s_{i+1} \cdots s_K / \alpha)$ iterations to find an $\alpha$-stationary point for the weight matrice at layer $i <K$.
\end{corollary}

For a DNN, the training loss  is often nonconvex, Lipschitz, but sometimes nonsmooth, e.g. when using ReLU activations. Therefore this corollary points out that a learning algorithm can require a significant number of  iterations to train an unnormalized model. The product $s_{i+1} \cdots s_K$ suggests that the weights at a deeper layer may require fewer iterations to converge. In the worst case, the weight of the first layer can require $\Omega( s_{2} \cdots s_K / \alpha)$ iterations. Figure~\ref{fig:MLP-weight-norms} suggests that $ s_{2} \cdots s_K$ can be extremely huge in practice. As a result, training an unnormalized DNN can be exponentially expensive in the worse case. 

\textit{Convergence rate:} For a normalized DNN, its loss could be much better (flatter) than that of  the unnormalized version. Such a property can help a training algorithm converge faster. Indeed, for nonconvex  Lipschitz  functions, \citep{davis2022GD,tian2022finite} present some algorithms to find a $\alpha$-stationary point with  $O(\| \ell, \vw \|_{Lip} / \alpha^2)$ function/gradient evaluations. Combining those results with Remark~\ref{remark-loss-Lipschitz}, we obtain the following.

\begin{corollary} [Convergence rate]
Consider $L_{\gH_{no}}$ and  $L_{\gH}$ being the empirical losses, defined for members of the family $\gH_{no}$ and $\gH$, respectively, where $\BN$ is used. Let $\sigma$ be defined as in Corollary \ref{cor-FCN-Lipschitz-constant-BN}. Then for any $\alpha>0$, an $\alpha$-stationary point for the weights at layer $i <K$ can be found with
\begin{itemize}
\item $O( s_{i+1} \cdots s_K / \alpha^2)$ function/gradient evaluations  of  $L_{\gH}$, or
\item $O( \sigma^{-K+i-1} s_{i+1} \cdots s_K / \alpha^2)$  function/gradient evaluations  of $L_{\gH_{no}}$.
\end{itemize}
\end{corollary}

When the input variances ($\sigma$) are large, this corollary suggests that a normalized network can be trained much more efficiently than its unnormalized version. There can be an exponential reduction ($\sigma^{-K+i-1}$) in the number of required iterations, especially for the weights at some first layers of a DNN. This is a significant advantage of a normalizer.

\subsection{Generalization ability of normalized DNNs}

We next discuss the consequences about generalization ability of normalized models. As indicated before, a normalizer can make the Lipschitz constant of a model much smaller. This is an important capacity control. This section shows why Lipschitz control can translate into generalization improvement. 

\subsubsection{Local Lipschitz continuity implies generalization}

To see generalization ability of a trained model $\vh$, the traditional way is to estimate (or bound) the expected loss $F(P, \vh) = \E_{\vx \sim P}[f(\vh,\vx)]$, where $P$ is the data distribution. A smaller expected loss implies that the model can better generalize on unseen data. Therefore, this quantity can help us to see the quality of a model $\vh$. Since  $P$ is unknown, we often rely on the \emph{empirical loss}  $F(\mD, \vh) =  \frac{1}{m} \sum_{\vx \in \mD} f(\vh,\vx)$ which is defined from a sample set $\mD$ of size $m$. We first show the following   in Appendix~\ref{app-Lipschitz-to-generalization}.

\begin{theorem}\label{thm-local-Lipschitz-generalization}
Consider a model $\vh$ and a dataset $\mD$ consisting of $m$ i.i.d. samples from distribution $P$.  Let $\bigcup_{i=1}^{N } \gX_i$ be a partition of the data space into $N $ disjoint nonempty subsets, $\mD_i = \mD \cap \gX_i$, $m_i = | \mD_i |$ as the number of samples  falling into $\gX_i$,  $m = \sum_{j=1}^N m_j$, and $\mT = \{ i \in [N ] : m_i > 0 \}$. Denote $\lambda_i = \frac{1}{m_i} \sum_{\vs \in \mD_i} \E_{\vx \sim  P}[\|\vx - \vs \|: \vx \in \gX_i]$ as the average distance between one sample in $\mD_i$ and one in $\gX_i$. 
For any  $\delta >0$, denote $g( \mD, \delta) = C (\sqrt{2} + 1) \sqrt{\frac{| \mT| \ln(2 N /\delta)}{m}} + \frac{ 2 C | \mT| \ln(2 N /\delta)}{m}$, where $C = \sup_{\vx} f(\vh,\vx)$. If the loss $f$ is $L_i$-Lipschitz continuous w.r.t. $\vx \in \gX_i$ for any $i \in \mT$, then the following holds  with probability at least $1-\delta$: 
\begin{equation}
\label{eq-thm-local-Lipschitz-generalization}
F(P, \vh) \le F(\mD, \vh) +  \sum_{i \in \mT} \frac{m_i}{m} \lambda_i L_i  +  g( \mD, \delta)
\end{equation}
\end{theorem}

This result tells that the test error of a model $\vh$ can be bounded (or estimated) by the right-hand side of (\ref{eq-thm-local-Lipschitz-generalization}). When the RHS  of (\ref{eq-thm-local-Lipschitz-generalization}) is small, model $\vh$ is guaranteed to have a small test error, implying that it can generalize well on unseen data. This theorem suggests that a model should fit the training set well while ensuring its loss to have small local Lipschitz constants at local areas of the data space. This suggests that a model should behave gently and does not change significantly w.r.t a small change of its input. As a result, this bound provides a guide to train a good model in practice.

Bound (\ref{eq-thm-local-Lipschitz-generalization}) depends only on the specific sample $\mD$ and model $\vh$ (but not the whole family $\gH$), and hence is both model-specific and data-dependent. Note that we only need the  assumption of local Lipschitzness on some small areas around the individual samples of $\mD$. As a result, it can help us to  analyze a large class of models. Furthermore, this bound can work even with non-Lipschitz functions, as indicated below.

\begin{corollary}[Discontinuity]\label{cor-local-Lipschitz-generalization}
Consider the setting and notations in Theorem \ref{thm-local-Lipschitz-generalization}. Let $\gX = \bigcup_{i=0}^{N } \gX_i$ where $\gX_0$ is the set of all instances so that the loss $f$ is not Lipschitz continuous at any $\vx \in \gX_0$. If $P(\gX_0)=0$ and $f$ is $L_i$-Lipschitz continuous w.r.t. $\vx \in \gX_i$ for any $i \in \mT, i>0$, then   with probability at least $1-\delta$: 
$F(P, \vh) \le F(\mD, \vh) +  \sum_{i \in \mT} \frac{m_i}{m} \lambda_i L_i  +  g( \mD, \delta).$
\end{corollary} 

This simple result provides a significant implication. It suggests that a model can generalize well although its loss function may be discontinuous at some points. Furthermore, a good generalization can be assured for the cases that the loss function is discontinuous in an uncountable set whose measure is 0. Those discontinuous cases can appear in practice.

\textit{Related model-dependent bounds:} \cite{hou2023instanceGen} used optimal transport to show that 
\begin{equation}
\label{eq-HKK}
F(P, \vh) \le F(\mD, \vh) + \gamma(L) +  \sqrt{{N}/{m}} + const., \tag{HKK}
\end{equation}
where $\gamma(L)= \sum_{i \in \mT} \frac{m_i}{m} \lambda_i \max\{1, L_i \}  + \max_i \lambda_i L \sqrt{\frac{\ln (4/ \delta)}{m}}$ and $L = \max \{1, \| f, \vx \|_{Lip}\} $, requiring the loss to be globally Lipschitz continuous. As pointed out before, such an assumption  makes this bound already more restrictive than ours.  Furthermore, the term $\sqrt{N/m}$ causes their bound to suffer from the curse of dimensionality, since $N = \lambda^{-O(n)}$ in the worst case, where $n$ is the input dimensionality, meanwhile our bound does not. Those reasons suggest that the bound (\ref{eq-HKK}) is loose and significantly inferior to ours. Furthermore, {their bound cannot be applied to the cases of discontinuity}, while our result applies well to those cases as shown by Corollary~\ref{cor-local-Lipschitz-generalization}.

\citet{than2025LocalRobustness} recently provided an extensive analysis about the connection between robustness and generalization ability of a model. Their theories suggest that a model should be locally robust at every small areas of the input space to ensure a high performance. Robustness in \citep{than2025LocalRobustness} refers to the change of the loss w.r.t. a change in the input, which is very different from the algorithmic robustness by \citep{xu2012robustnessGeneralize,kawaguchi22RobustGen}. Our bound (\ref{eq-thm-local-Lipschitz-generalization}) has the same nature with those by \cite{than2025LocalRobustness}, which are model-specific. In fact, our bound is a corollary of the results by \cite{than2025LocalRobustness}, since we focus on Lipschitz continuity (a specific property of the loss function). Nonetheless, Corollary~\ref{cor-local-Lipschitz-generalization} extends the horizon to analyze a wider class of practical models.

\subsubsection{Consequences for normalized DNNs}

The basic property of composite function tells that $\| f, \vx \|_{Lip} \le L_f \| \vh, \vx \|_{Lip}$ when $f$ is $L_f$-Lipschitz continuous w.r.t $\vh$. Combining this fact with Corollary~\ref{cor-FCN-Lipschitz-constant-BN} and Theorem~\ref{thm-local-Lipschitz-generalization} results in the followings.

\begin{corollary}  \label{cor-normalized-DNN-Generalization}
Consider a neural network $\vh$ defined in Definition~\ref{def-FCN}, and its normalized version $\vh_{no}$. Given the assumption in  Theorem~\ref{thm-local-Lipschitz-generalization} and notations in Corollary~\ref{cor-FCN-Lipschitz-constant-BN} and Corollary~\ref{cor-FCN-Lipschitz-constant-LN}, assume further that loss $f(\vh, \cdot)$ is $L_f$-Lipschitz continuous w.r.t $\vh$. The followings hold with probability at least $1-\delta$: 
{\small \begin{eqnarray}
\label{eq-thm-Lipschitz-generalization-DNN}
F(P, \vh) - F(\mD, \vh) &\le&   g( \mD, \delta)  + \omega(\vh)  \\
\nonumber
F(P, \vh_{no}) - F(\mD, \vh_{no}) &\le&  g( \mD, \delta) + \omega(\vh_{no}) \sigma^{-K} \;\;\;\;\; \;\;\;\;\;\;\;  (\BN) \\
\nonumber
 F(P, \vh_{no}) - F(\mD, \vh_{no}) &\le&   g( \mD, \delta) + \omega(\vh_{no})  \prod_{k=1}^{K} 1/{\sigma_{k}}   \;\;\;\;\;  (\LN)
\end{eqnarray}}
where $\omega(\vh) = L_f P_w(\vh) \sum_{i \in \mT} \frac{m_i}{m} \lambda_i$.
\end{corollary}

This theorem suggests that $\BN$ and $\LN$ can significantly reduce the generalization error of a DNN. Without normalization, quantity $\omega(\vh)$ can be of order $P_w(\vh)$ which can be extremely large in practice.  In practice, the product $\prod_{k=1}^{K}  1/{\sigma_{k}}$ can be exponentially small, as pointed out in Subsection~\ref{subsec-smoothen} and Figure~\ref{fig:variance-before-normalization}, suggesting that normalization operators can significantly improve generalization ability of a normalized DNN.

\textit{Comparison with existing bound for normalized DNNs:} To the best of our knowledge, theoretical results on generalization ability of normalized DNNs still remain limited. The only existing bound for normalized DNNs was developed by \cite{lyu22understandingNorm}. For a function $\vh(\vx, \mW)$ with parameters $\mW$, \cite{lyu22understandingNorm} can incorporate the \textit{Spherical sharpness} $\tau$ of the training loss to show that 
{\small \begin{equation}
\label{eq-LLA}
\E_{\vu} \left[F(P, \vh(\vx, \mW + \vu))\right]   \le  F(\mD, \vh) + \frac{\tau\mu^2 }{2} + \frac{O(\ln m)}{\mu \sqrt{m}} + const., \tag{LLA}
\end{equation}}
where each element of $\vu$ follows the normal distribution with mean 0 and variance $\mu^2/n$. This result however is limited, as it only applies to a \textit{smoothed} version of the test error. Bridging this to the true test error would require setting $\mu \rightarrow 0$, which in turn causes the third term of (\ref{eq-LLA}) to diverge. This reveals a critical trade-off: achieving a more accurate error estimate (via a smaller $\mu$) necessitates a quadratic increase in sample size. This indicates that their bound is less sample-efficient than the one we propose. Furthermore, a more significant drawback is that bound (\ref{eq-LLA}) cannot capture the cumulative benefits of multiple normalization layers, whereas our bound can.

In other work, \cite{cisneros25WN} introduced a new bound that successfully demonstrates the generalization benefit of Weight Normalization (WN). Their analysis, however, is confined to a specific architecture where each layer's normalized weight matrix has a Frobenious norm of 1. While their resulting $O(\rho_1^2/\sqrt{m})$ bound (where $\rho_1$ is the radius of the last layer's weight domain) effectively highlights WN's advantage, it is not generalizable. It cannot be applied to other normalization methods. Our bound, in sharp contrast, is broadly applicable to many influential normalizers, as shown in Corollary~\ref{cor-normalized-DNN-Generalization}.

\section{Conclusion} \label{sec-Conclusion}

This work establishes a unified theoretical framework for understanding the role of normalization in DNNs through the lens of Lipschitz-based capacity control. We have shown that unnormalized networks can possess exponentially large or small Lipschitz constants --- both with respect to weights and inputs --- resulting in uncountably many configurations prone to overfitting, underfitting, or gradient instability. These findings formally connect the empirical challenges of sharp loss landscapes and unstable gradients to the intrinsic capacity of the network.

In contrast, normalization operations  introduce an explicit and exponential reduction in the network's Lipschitz constant. This reduction simultaneously smooths the optimization landscape and constrains the functional capacity of the model, leading to faster convergence and improved generalization. Our theoretical analysis further bridges the gap between local Lipschitz continuity and generalization performance, providing a principled explanation for why normalized DNNs generalize well even when global Lipschitz continuity does not hold.

Beyond offering a rigorous explanation for the empirical success of normalization methods, our results highlight a broader principle: normalization acts as an implicit but powerful form of regularization that scales exponentially with network depth. The framework developed here lays the groundwork for future exploration of normalization-inspired designs in diverse architectures  and for developing new normalization schemes that explicitly optimize capacity control for both stability and generalization.

Nonetheless, our work has certain limitations. First, our theoretical analysis omits the effect of the scale and shift parameters commonly included in normalization layers. These parameters are known to enhance the representational and approximation capacity of DNNs, yet their precise role in the generalization ability of trained models remains unclear. Second, while our analysis primarily focuses on feedforward networks, extending these results to more complex architectures (e.g.,  transformer-based models) may require substantial additional effort and new theoretical tools.


\acks{We would like to thank Bum Jun Kim for pointing out some typos or inconsistencies.}


\newpage

\appendix


\section{Properties of some normalizers} \label{app-NO}

\subsection{Properties of Batch norm} \label{app-BN}

\begin{proof}[Proof of Lemma \ref{lem-BN-Jacobian-norm}]

By definition, we have $\partial \BN(x_k, \epsilon) / \partial x_k = 1/{\sqrt{\sigma_k^2 + \epsilon}}$, while $\partial \BN(x_k, \epsilon) / \partial x_j = 0$ for any $j \neq k$. So $\left\|\frac{\partial \BN(\vx, \epsilon)}{\partial \vx} \right\|  = \| 1/{\boldsymbol{\sigma}} \|$. Therefore 

\[\| \BN(\vx, \epsilon) \|_{Lip} = \sup_{\vx \in \R^n} \left\| \frac{\partial \BN(\vx, \epsilon)}{\partial \vx} \right\|  
= \sup_{\vx \in \R^n} \| 1/{\boldsymbol{\sigma}} \| = \| 1/{\boldsymbol{\sigma}} \|\]
which completes the proof.
\end{proof}

\subsection{Properties of Layer norm} \label{app-LN}

We first have the following observation. 

\begin{lemma}\label{lem-LN-Jacobian}
Given $\epsilon>0$, and $\vx \in \R^n$, we have 
\[
\frac{\partial \LN(\vx, \epsilon)}{\partial \vx} = \frac{1}{\sqrt{\sigma^2_n + \epsilon}} \left(\mI - \frac{\vy \vy^T}{\|\vy\|^2 + n \epsilon}\right) \left(\mI - \frac{1}{n} \bm{1}\right)
\]
where $\vy = (\mI - \frac{1}{n} \bm{1})\vx$, $\bm{1}$ is the $n \times n$ matrix of $1$'s,  and $\mI$ is the identity matrix.
\end{lemma}

\begin{proof}
We can explicitly write $y_i = x_i - \mu$ for the $i$-th element in $\vy$. The $i$-th output of LN is $z_i = \frac{y_i}{\sqrt{\sigma_n^2 + \epsilon}}$. 
Recall that $\sigma_n^2 + \epsilon = \frac{1}{n} \sum_{j} y_j^2 + \epsilon = \frac{1}{n} \|\vy\|^2 + \epsilon$.
The partial derivative with respect to the centered vector $\vy$ is:
\begin{eqnarray}
\frac{\partial z_i}{\partial y_t} 
&=& \frac{\delta_{it} \sqrt{\frac{1}{n}\|\vy\|^2 + \epsilon} - y_i \left( \frac{1}{2\sqrt{\frac{1}{n}\|\vy\|^2 + \epsilon}} \cdot \frac{2 y_t}{n} \right)}{\frac{1}{n}\|\vy\|^2 + \epsilon} \\
&=& \frac{1}{\sqrt{\sigma_n^2 + \epsilon}} \left( \delta_{it} - \frac{y_i y_t}{n(\sigma_n^2 + \epsilon)} \right) \\
&=& \frac{1}{\sqrt{\sigma_n^2 + \epsilon}} \left( \delta_{it} - \frac{y_i y_t}{\|\vy\|^2 + n\epsilon} \right)
\end{eqnarray}
where $\delta_{it}=1$ for $i=t$ and $\delta_{it}=0$ for $i \ne t$. In matrix form (Jacobian w.r.t $\vy$):
\begin{equation}
\frac{\partial \LN}{\partial \vy} = \frac{1}{\sqrt{\sigma^2_n + \epsilon}} \left(\mI - \frac{\vy \vy^T}{\|\vy\|^2 + n \epsilon}\right)
\end{equation}
Note that $\frac{\partial \vy}{\partial \vx} = \mI - \frac{1}{n}\bm{1}$. By the chain rule $\frac{\partial \LN}{\partial \vx} = \frac{\partial \LN}{\partial \vy} \frac{\partial \vy}{\partial \vx}$, we obtain:
\begin{equation}
\nonumber
\frac{\partial \LN(\vx, \epsilon)}{\partial \vx} = \frac{1}{\sqrt{\sigma^2_n + \epsilon}} \left(\mI - \frac{\vy \vy^T}{\|\vy\|^2 + n \epsilon}\right) \left(\mI - \frac{1}{n} \bm{1}\right)
\end{equation}
\end{proof}

\begin{proof}[Proof of Lemma \ref{lem-LN-Jacobian-norm}]
From Lemma \ref{lem-LN-Jacobian}, the Jacobian of the Layer Normalization function is given by:
\begin{equation}
\label{eq-Jacobian-Matrix}
\frac{\partial \LN(\vx, \epsilon)}{\partial \vx} = \frac{1}{\sqrt{\sigma^2_n + \epsilon}} \left(\mI - \frac{\vy \vy^T}{\|\vy\|^2 + n \epsilon}\right) \left(\mI - \frac{1}{n} \bm{1}\right)
\end{equation}
where $\vy = (\mI - \frac{1}{n} \bm{1})\vx$ is the centered input vector. We seek to bound the spectral norm of this Jacobian.

Using the sub-multiplicative property of the spectral norm, we have:
\begin{eqnarray}
\nonumber
\left\| \frac{\partial \LN(\vx, \epsilon)}{\partial \vx} \right\| 
&\le& \frac{1}{\sqrt{\sigma^2_n + \epsilon}} \left\| \mI - \frac{\vy \vy^T}{\|\vy\|^2 + n \epsilon} \right\| \left\| \mI - \frac{1}{n} \bm{1} \right\|
\end{eqnarray}

We analyze the two matrix terms separately:
\begin{enumerate}
    \item Let $\mP_1 = \mI - \frac{1}{n} \bm{1}$. This is the centering matrix, which projects vectors onto the subspace orthogonal to the vector of all ones. As a projection matrix (and assuming $n > 1$), its eigenvalues are either 0 or 1. Thus, its spectral norm is $\|\mP_1\| = 1$.
    
    \item Let $\mP_2 = \mI - \frac{\vy \vy^T}{\|\vy\|^2 + n \epsilon}$. 
    \begin{itemize}
        \item If $\epsilon = 0$, $\mP_2$ is a standard projection matrix orthogonal to $\vy$, and $\|\mP_2\| = 1$.
        \item If $\epsilon > 0$, $\mP_2$ scales the component of a vector parallel to $\vy$ by a factor of $1 - \frac{\|\vy\|^2}{\|\vy\|^2 + n \epsilon} = \frac{n \epsilon}{\|\vy\|^2 + n \epsilon} \in (0, 1)$, while leaving components orthogonal to $\vy$ unchanged (eigenvalue 1). Thus, the maximum singular value is still 1, so $\|\mP_2\| \le 1$.
    \end{itemize}
\end{enumerate}

Substituting these bounds back into the inequality:
\begin{eqnarray}
\left\| \frac{\partial \LN(\vx, \epsilon)}{\partial \vx} \right\| 
&\le& \frac{1}{\sqrt{\sigma^2_n + \epsilon}} \cdot 1 \cdot 1 \\
&=& \frac{1}{\sqrt{\sigma^2_n + \epsilon}}
\end{eqnarray}

Finally, taking the supremum over all $\vx$:
\begin{eqnarray}
\| \LN \|_{Lip} &=& \sup_{\vx \in \R^n} \left\| \frac{\partial \LN(\vx, \epsilon)}{\partial \vx} \right\| \\
&\le& \sup_{\vx \in \R^n} \frac{1}{\sqrt{\sigma^2_n(\vx) + \epsilon}} \\
&=& \frac{1}{\sigma}
\end{eqnarray}
which completes the proof.
\end{proof}

\subsection{Properties of Group norm} \label{app-GN}

\begin{proof}[Proof of Lemma \ref{lem-GN-Jacobian-norm}]

Denote $\vx_k$ be the part of a given $\vx$ correspoding to the index set $\mS_k$. Observe that $\frac{\partial \GN(\vx, \epsilon)}{\partial \vx}$ can be reorganized as

\begin{eqnarray}
\frac{\partial \GN(\vx, \epsilon)}{\partial \vx} &=& \bigcup_{k=1}^C \frac{\partial \GN(\vx, \epsilon)}{\partial \vx_k}
\end{eqnarray}
Therefore
\begin{eqnarray}
\label{app-eq-GN-02}
\left\| \frac{\partial \GN(\vx, \epsilon)}{\partial \vx} \right\| 
&\le&   \sum_{k=1}^C  \left\| \frac{\partial \GN(\vx, \epsilon)}{\partial \vx_k} \right\|
\end{eqnarray}

Remember that for each group $\mS_k$, the $\GN$ operator is in fact $\LN$. Hence, $\frac{\partial \GN(\vx, \epsilon)}{\partial \vx_k} = \frac{\partial \LN(\vx_k, \epsilon)}{\partial \vx_k}$.  The proof of Lemma~\ref{lem-LN-Jacobian-norm} implies $\left\| \frac{\partial \LN(\vx_k, \epsilon)}{\partial \vx_k} \right\| \le \frac{1}{\sqrt{\sigma_{k}^2(\vx) + \epsilon}}  $. Combining this with (\ref{app-eq-GN-02}) will arrive at 

\[\left\| \frac{\partial \GN(\vx, \epsilon) }{\partial \vx} \right\|  \le    \sum_{k=1}^C   \frac{1}{\sqrt{\sigma_{k}^2(\vx) + \epsilon}}  \]
Therefore 
\begin{eqnarray}
\sup_{\vx \in \R^n} \left\| \frac{\partial \GN(\vx, \epsilon) }{\partial \vx} \right\| 
&\le& \sup_{\vx \in \R^n}   \sum_{k=1}^C   \frac{1}{\sqrt{\sigma_{k}^2(\vx) + \epsilon}}  \\ 
&\le& \sum_{k=1}^C   \frac{1}{\sqrt{\sigma_{k}^2 + \epsilon}} 
\end{eqnarray}
which proves the first statement.

The second statement is a straightforward consequence of the first result, completing the proof.
\end{proof}

\section{DNNs with normalizations} \label{app-DNN-proofs-normalization}

\begin{proof}[Proof of Lemma \ref{lem-Lipschitz-FCN}]

$\vh$ can be expressed as $\vh = \vh_K = g_K(\mW_{K}\vh_{K-1})$. Using a basic property of Lipschitz continuity for composite functions, we have
\begin{eqnarray}
\|\vh,\vx \|_{Lip} 
&\le& \| g_{K} \|_{Lip} \| \mW_{K} \vh_{K-1},\vx \|_{Lip} \\
&\le& \| \mW_{K} \| \| \vh_{K-1},\vx \|_{Lip} \\
&\le&  \| \mW_{K} \| \| \vh_{K-1},\vx \|_{Lip} \\
&\cdots& \\
&\le& \prod_{k=1}^{K}   \| \mW_{k} \|   
\end{eqnarray} 
completing the proof.
\end{proof}

\begin{proof}[Proof of Theorem \ref{lem-Lipschitz-FCN-normalize}]

Consider an $\vh_{no}$. Using a basic property of Lipschitz continuity for composite functions, we have
\begin{eqnarray}
\nonumber
\|\vh_{no},\vx \|_{Lip} 
&\le& \| g_{K} \|_{Lip} \| \NO_{K} \|_{Lip} \| \mW_{K} \vh_{K-1},\vx \|_{Lip} \\
\nonumber
&\le& \| g_{K} \|_{Lip} \| \NO_{K} \|_{Lip} \| \mW_{K} \| \| \vh_{K-1},\vx \|_{Lip} \\
\nonumber
&\le&  \| \mW_{K} \| \| \NO_{K} \|_{Lip} \| \vh_{K-1},\vx \|_{Lip} \\
\nonumber
&\cdots& \\
&\le& \prod_{k=1}^{K}   \| \mW_{k} \| \| \NO_{k} \|_{Lip}  
\end{eqnarray} 
completing the proof.
\end{proof}

\begin{proof}[Proof of Theorem \ref{thm-Lipschitz-FCN-normalize-weight}]

Consider a model $\vh_{no} = \vh_K = g_K(\NO_{K}(\vy_K))$. A basic property of Lipschitz continuity for composite functions suggests that \\ $\left\| \vh_{no}, \vy_K \right\|_{Lip} \le \| g_{K} \|_{Lip} \| \NO_{K} \|_{Lip} \le \| \NO_{K} \|_{Lip}$. 

For any index $i < K$, observe that \\
$\left\| \vh_{no}, \vy_i \right\|_{Lip} = \left\| \vh_K, \vy_i \right\|_{Lip}$
\begin{flalign*}
&\le \| g_{K} \|_{Lip} \| \NO_{K} \|_{Lip} \| \mW_{K} \vh_{K-1}, \vy_i  \|_{Lip} \\
&\le \| g_{K} \|_{Lip} \| \NO_{K} \|_{Lip} \| \mW_{K} \| \| \vh_{K-1} , \vy_i \|_{Lip} \\
&\le  \| \mW_{K} \| \| \NO_{K} \|_{Lip} \| \vh_{K-1} , \vy_i \|_{Lip} \\
&\cdots \\
&\le (\prod_{k=i+2}^{K}   \| \mW_{k} \| \| \NO_{k} \|_{Lip}  )  \| \vh_{i+1}, \vy_i  \|_{Lip} \\
&\le (\prod_{k=i+2}^{K}   \| \mW_{k} \| \| \NO_{k} \|_{Lip}  ) \| \NO_{i+1} \|_{Lip} \| \mW_{i+1} \| \| \vh_{i}, \vy_i  \|_{Lip} \\
&\le (\prod_{k=i+1}^{K}   \| \mW_{k} \| \| \NO_{k} \|_{Lip}  ) \| \NO_{i} \|_{Lip} \\
&\le (\prod_{k=i+1}^{K}   \| \mW_{k} \|) \prod_{k=i}^{K} \| \NO_{k} \|_{Lip} 
\end{flalign*} 

Similarly, we next consider the Lipschitz constant w.r.t. the weights. Note  that $\left\| \vh_{no}, \mW_i \right\|_{Lip} \le \left\| \vh_{no}, \vy_i \right\|_{Lip}  \| \vy_{i}, \mW_{i} \|_{Lip}$. Furthermore, $\| \vy_{i}, \mW_{i} \|_{Lip} =  \sup_{\vh_{i-1}} \|\vh_{i-1}\| = A_{i-1}$ due to $\vy_{i} = \mW_{i} \vh_{i-1}$. As a result, $\left\| \vh_{no}, \mW_i \right\|_{Lip} \le A_{i-1} \left\| \vh_{no}, \vy_i \right\|_{Lip} $,
completing the proof.
\end{proof}

\section{Unnormalized DNNs} \label{app-DNN-proofs-no-norm}

\begin{proof}[Proof of Theorem \ref{thm-Lipschitz-FCN-lower-bound}]

Let $n_i$ be the width of layer $i$ of a DNN in $\gH$, where $n_0 = n$. For simplicity, we denote $\mI_i$ as the following matrix of size $n_{i} \times n_{i-1}$:
\begin{equation}
\label{app-thm-Lipschitz-FCN-lower-bound-01}
\mI_i = \begin{cases}
(A_i, \textbf{0}_i) & \text{ if } n_{i-1} > n_i, \\
A_i  & \text{ if } n_{i-1} = n_i, \\
\left(
\begin{matrix}
A_i \\
\textbf{0}_i
\end{matrix}\right) & \text{ otherwise }
\end{cases}
\end{equation} 
where $A_i$ is the identity matrix of size  $\min\{n_i, n_{i-1}\}$, $\textbf{0}_i$ is the matrix of zeros with an appropriate size. For a given vector $\vv$ of size $n_{i-1}$, $\vv^{(n_i)}$ is the truncated/zero-padded vector of size  $n_i$. One can easily observe that $\vv^{(n_i)} = \mI_i \vv$ and $ReLU(\vv^{(n_i)}) = (ReLU(\vv))^{(n_i)}$.

Consider the member $\vh^*(\vx) \equiv \vh_K(\vx)$ whose the weight matrix at layer $i$ is $\mW_i = a_i \mI_i$. Observe that, for any $k >0$, 
\begin{align}
ReLU(\vh_{k}) &= ReLU(ReLU(\mW_k \vh_{k-1})) &\\
&= ReLU(\mW_k \vh_{k-1}) &\\
&= ReLU(a_k \mI_k \vh_{k-1}) &\\
&= a_k ReLU( \vh_{k-1}^{(n_k)}) &\\
\label{app-thm-Lipschitz-FCN-lower-bound-02}
&= a_k (ReLU( \vh_{k-1}))^{(n_k)} &
\end{align}
Then it follows that
\begin{align}
\vh_{K} &= ReLU(\mW_K \vh_{K-1}) & \\
&= ReLU(a_K \mI_K \vh_{K-1}) & \\
&= a_K ReLU( \vh_{K-1}^{(n_K)}) & \\
&= a_K (ReLU( \vh_{K-1}))^{(n_K)} & \\
&= a_K a_{K-1} \left(ReLU( \vh_{K-2})\right)^{(n_{K-1},n_K)}  & (\text{Due to (\ref{app-thm-Lipschitz-FCN-lower-bound-02})}) \\
\nonumber
&\cdots & \\
\label{app-thm-Lipschitz-FCN-lower-bound-03}
&= a_K a_{K-1}\cdots a_{1} \left(ReLU( \vh_{0})\right)^{(n_{1},...,n_K)}  & (\text{By  induction}) \\
&= a_K a_{K-1}\cdots a_{1} \left(ReLU( \vx)\right)^{(n_{1},...,n_K)}  & \\
\label{app-thm-Lipschitz-FCN-lower-bound-04}
&= a_K a_{K-1}\cdots a_{1} ReLU\left(\vx^{(n_{1},...,n_K)} \right) &
\end{align}
where we have used the fact that $ \vh_0 \equiv \vx$. 

By definition, the operation $\vv^{(m)}$ is the truncation/zero-padding operator which keep a part of vector $\vv$. Therefore $\|\vx^{(m)}, \vx\|_{Lip} = 1$ for any integer $m$. Furthermore, ReLU is 1-Lipschitz continuous in the input. These suggest that $\| ReLU\left( \vx^{(n_{1},...,n_K)} \right), \vx\|_{Lip} = \|  \vx^{(n_{1},...,n_K)}, \vx\|_{Lip} = 1$. Combining this with (\ref{app-thm-Lipschitz-FCN-lower-bound-04}), we obtain $\| \vh^*,\vx \|_{Lip} = a_K a_{K-1}\cdots a_{1}$, completing the proof.
\end{proof}

\begin{proof}[Proof of Theorem \ref{thm-Lipschitz-FCN-lower-bound-weight}]

Consider the member $\vh^* \equiv \vh_K$ whose the weight matrix at layer $k$ is $\mW_k = a_k \mI_k$ for all $k > i$, where $ \mI_k$ is defined in Equation~(\ref{app-thm-Lipschitz-FCN-lower-bound-01}). Following the same arguments as the proof of Theorem~\ref{thm-Lipschitz-FCN-lower-bound} before, we can obtain a similar result as Equation~(\ref{app-thm-Lipschitz-FCN-lower-bound-03}):
\begin{align}
\vh_{K} &= a_K a_{K-1}\cdots a_{i+1} \left(ReLU( \vh_{i})\right)^{(n_{i+1},...,n_K)}  &  \\
&= a_K a_{K-1}\cdots a_{i+1} \left(ReLU( ReLU( \vy_{i}))\right)^{(n_{i+1},...,n_K)}  &  \\
&= a_K a_{K-1}\cdots a_{i+1} \left(ReLU( \vy_{i})\right)^{(n_{i+1},...,n_K)}  &  \\
\label{app-thm-Lipschitz-FCN-lower-bound-weight-01}
&= a_K a_{K-1}\cdots a_{i+1} ReLU\left(\vy_{i}^{(n_{i+1},...,n_K)} \right) &
\end{align}
Therefore \\
$\| \vh_K, \vy_{i}\|_{Lip} $
\begin{align}
&= a_K a_{K-1}\cdots a_{i+1} \left\| ReLU\left(\vy_{i}^{(n_{i+1},...,n_K)} \right), \vy_i\right\|_{Lip} & \\
&= a_K a_{K-1}\cdots a_{i+1} \left\| \vy_{i}^{(n_{i+1},...,n_K)}, \vy_i\right\|_{Lip} & \\
&= a_K \cdots a_{i+1} &
\end{align}
 
Similarly, \\
$\| \vh_K, \mW_{i}\|_{Lip} $
\begin{align}
&= a_K a_{K-1}\cdots a_{i+1} \left\| ReLU\left(\vy_{i}^{(n_{i+1},...,n_K)} \right), \mW_i\right\|_{Lip} & \\
&= a_K a_{K-1}\cdots a_{i+1} \left\| \vy_{i}^{(n_{i+1},...,n_K)} , \mW_i\right\|_{Lip} & \\
&= a_K a_{K-1}\cdots a_{i+1} \left\| (\mW_i\vh_{i-1})^{(n_{i+1},...,n_K)} , \mW_i\right\|_{Lip} & \\
&= a_K a_{K-1}\cdots a_{i+1} \left\| \vh_{i-1}^{(n_{i+1},...,n_K)}\right\|&
\end{align}
where we have used the facts that $\left\| (\mW\vx , \mW\right\|_{Lip} = \|\vx\|$ and $\vx^{(n_{i+1},...,n_K)}$ is a series of truncation or zero-padding to $\vx$. This completes the proof.
\end{proof}

\begin{proof}[Proof of Theorem \ref{thm-FCN-gradient-weight}]

Consider  $\vh^*(\vx) = g(\mW \vh_{K}(\vx))$ for $\mW= c \mathbf{1}^\top$ and $\vh_{K}$ whose the weight matrix at layer $k$ is $\mW_k = a_k \mI_k$ for all $k > i$, where $ \mI_k$ is defined in Equation~(\ref{app-thm-Lipschitz-FCN-lower-bound-01}). Following the same arguments as the proof of Theorem~\ref{thm-Lipschitz-FCN-lower-bound-weight} before, we can obtain a similar result as Equation~(\ref{app-thm-Lipschitz-FCN-lower-bound-weight-01}):
\begin{align}
\vh_K &= a_K \cdots a_{i+1} ReLU\left(\vy_{i}^{(n_{i+1},...,n_K)} \right)  \\ 
\label{app-thm-FCN-gradient-01}
&= a_K \cdots a_{i+1} ReLU\left((\mW_i\vh_{i-1})^{(n_{i+1},...,n_K)} \right) 
\end{align}

Then it follows that for  $t$-th row of $\mW_{i}$:
\begin{align}
\frac{\partial  L}{\partial \mW_{it}} 
 &= \frac{\partial  }{\partial \mW_{it}}\left(  \frac{1}{m}\sum_{\vx \in \mD} f(\vh, \vx) \right) \\
&=  \frac{1}{m}\sum_{\vx \in \mD} \frac{\partial  (f \circ g)}{\partial u}   \frac{\partial  u}{\partial \vh_{K}}  \frac{\partial  \vh_{K}(\vx)}{\partial \mW_{it}} \\
\label{app-thm-FCN-gradient-02}
&=  \frac{1}{m}\sum_{\vx \in \mD}  \frac{\partial  (f \circ g)}{\partial u}  \mW \frac{\partial  \vh_{K}(\vx)}{\partial \mW_{it}}
\end{align}
where $u(\vx) = \mW \vh_{K}(\vx)$.

Consider vector $\vv = ReLU\left((\mW_i\vh_{i-1})^{(n_{i+1},...,n_K)} \right)$, which is the series of truncation/zero-padding to $\mW_i\vh_{i-1}$ and then followed by $ReLU$ operation. As a result the element at row $k$ is
{\small \begin{align}
\nonumber
v_k = \begin{cases}
\mW_{ik}\vh_{i-1}  & \text{ if } k \le \min \{n_{i+1},...,n_K\} \text{ and } \mW_{ik} \vh_{i-1} \ge 0 \\
0 & \text{ otherwise}
\end{cases}
\end{align}}
As a result 
{\small \begin{align}
\frac{\partial  v_k}{\partial \mW_{it}} 
=  \begin{cases}
 \vh_{i-1}^\top  & \text{ if } k=t, t \le \min \{n_{i+1},...,n_K\} \text{ and } \mW_{ik} \vh_{i-1} \ge 0 \\
\mathbf{0} & \text{ otherwise}
\end{cases}
\end{align}}
This suggests that $\mW \frac{\partial  \vv}{\partial \mW_{it}} = c \mathbf{1}^\top \frac{\partial  \vv}{\partial \mW_{it}} = c \vh_{i-1}^\top$ for any index $t \le \min \{n_{i+1},...,n_K\}$  satisfying  $\mW_{it} \vh_{i-1} \ge 0$. Combining this with (\ref{app-thm-FCN-gradient-02}) and the fact that $\vh_K(\vx) = a_K \cdots a_{i+1} \vv(\vx)$ due to (\ref{app-thm-FCN-gradient-01}) , we obtain 
\begin{align}
\nonumber
\frac{\partial  L}{\partial \mW_{it}} =  \frac{c a_K \cdots a_{i+1} }{m}\sum\limits_{\vx \in \mD}   \frac{\partial  (f \circ g)}{\partial u}  \gI_{it} (\vh_{i-1}(\vx))^\top 
\end{align}
for any $t \le \min \{n_{i+1},...,n_K\}$, completing the proof.
\end{proof}

\section{From local Lipschitz continuity to generalization} \label{app-Lipschitz-to-generalization}

\begin{proof}[Proof of Theorem \ref{thm-local-Lipschitz-generalization}]

Theorem 5 in \cite{than2025LocalRobustness} shows that, with probability at least $1-\delta$, we have 
\begin{eqnarray}
\label{eq-error-decomposition-01}
F(P, \vh) \le  F(\mD, \vh) +   {  \sum_{i \in \mT}  \frac{m_i}{m} \bar{\epsilon}_i(\vh) }  + g(\mD, \delta) 
\end{eqnarray}
where $\bar{\epsilon}_i(\vh) = \frac{1}{m_i} \sum_{\vs \in \mD_i} \E_{\vx \in \gX_i} | f(\vh,\vx) -  f(\vh,\vs) |$ represents the sensitivity of model $\vh$ in area $\gX_i$ and $\mD_i = \mD \cap \gX_i$.

Since $f(\vh, \vx)$ is  $L_i$-Lipschitz continuous on $\gX_i$, for any $i \in \mT$, we have 
$  | f(\vh, \vx) -  f(\vh, \vs) | \le L_i \| \vx -  \vs \| $ for all $\vx, \vs \in \gX_i $.  Therefore
\begin{eqnarray}
\nonumber
 \E_{\vx \in \gX_i} | f(\vh, \vx) -  f(\vh, \vs) | &\le& \E_{\vx \in \gX_i} [L_i \| \vx -  \vs \|] \\
 \nonumber
  &\le& L_i \E_{\vx \in \gX_i} [\| \vx -  \vs \|]
\end{eqnarray}
It suggests that $\bar{\epsilon}_i(\vh) \le \frac{1}{m_i} \sum_{\vs \in \mD_i} L_i \E_{\vx \in \gX_i} [\| \vx -  \vs \|] = L_i \lambda_i$. Combining this with (\ref{eq-error-decomposition-01}) completes the proof.
\end{proof}

\begin{proof}[Proof of Corollary \ref{cor-local-Lipschitz-generalization}]
Denote  $\bar{\gX}_0 = \gX \setminus \gX_0$, $a_0(\vh) = \E_{\vx \sim  P} [f(\vh,\vx) : \vx \in \gX_0]$, $\bar{a}_0(\vh) = \E_{\vx \sim  P} [f(\vh,\vx) : \vx \in \bar{\gX}_0]$. We can decompose the expected loss as $F(P, \vh) = P(\gX_0) a_0(\vh) + P(\bar{\gX}_0) \bar{a}_0(\vh) = \bar{a}_0(\vh)$, due to $P(\gX_0) =0$. Applying Theorem~\ref{thm-local-Lipschitz-generalization} for  $\bar{a}_0(\vh)$ will complete the proof.
\end{proof}

\section{Settings and more empirical evaluations} \label{app-empirical-evaluations}

In this section, we provide some empirical evaluations about the behavior of some DNNs, when using  $\BN$. 

\subsection{Experimental settings}

\textit{Dataset:} CIFAR10

\textit{Network architecture:} ResNet18, EfficientNet-B3 and a 10-layer ReLU network. For the ReLU network, we adopt a feedforward architecture consisting of a repeated ``Linear $\rightarrow$ ReLU'' pattern for a total of 10 layers. Each hidden layer contains 512 units, and the final output layer has 10 neurons corresponding to the 10 classes in the CIFAR10 dataset.

\textit{Training:} Those models are trained by standard Adam optimizer with learning rate of 0.001, batchsize=128, 100 epochs and weight decay of $10^{-4}$. Before training, data normalization is used  so that every training sample $\vx$ satisfies $\| \vx \| \in [0, 1]$. Those models are initialized with `He normal initialization'.

\subsection{Bahaviors of EfficientNet}

\begin{figure*}[t]
	\begin{minipage}[b]{0.74\textwidth} 
        \centering
        \includegraphics[width=.48\textwidth]{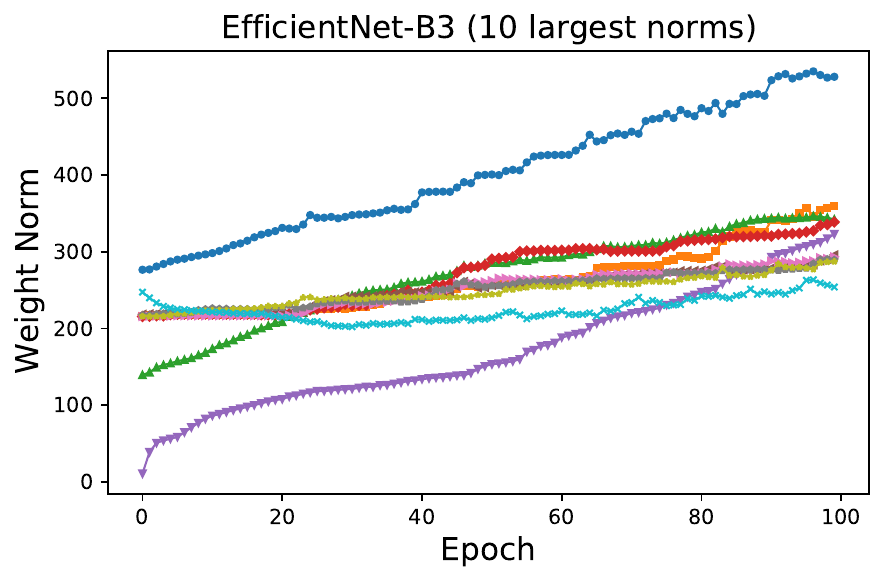} 
         \includegraphics[width=.48\textwidth]{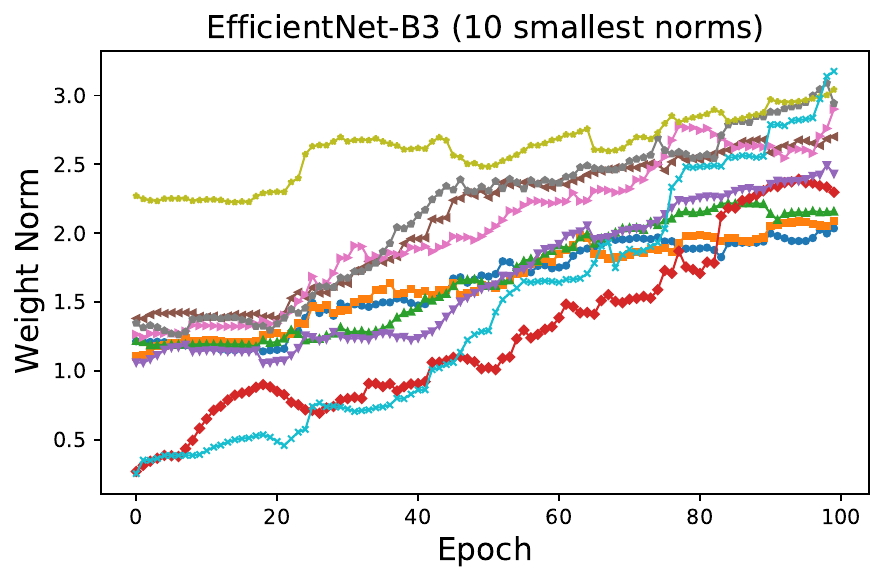} 
    \end{minipage}
     \begin{minipage}[b]{0.23\textwidth} 
             \centering
        \includegraphics[width=\textwidth]{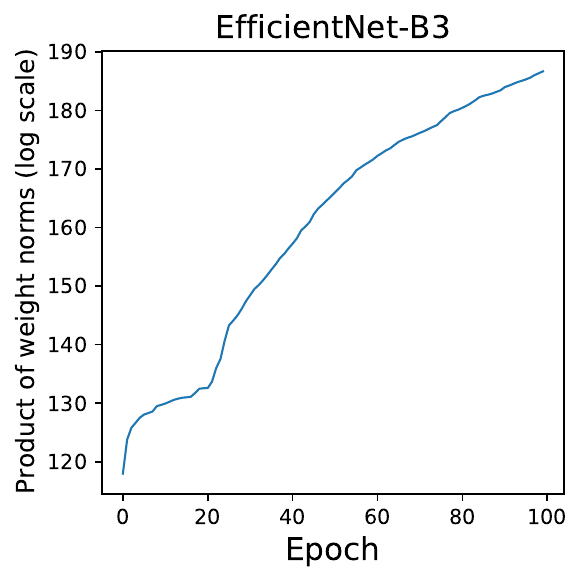}
         \end{minipage}
            
        \begin{minipage}[b]{0.74\textwidth} 
                \centering
		        \includegraphics[width=0.48\textwidth]{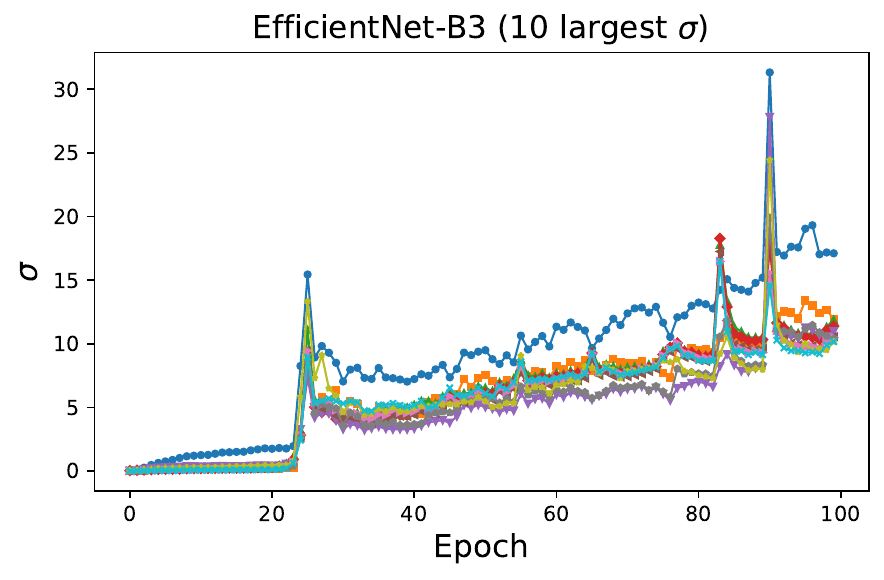}
		        \includegraphics[width=.48\textwidth]{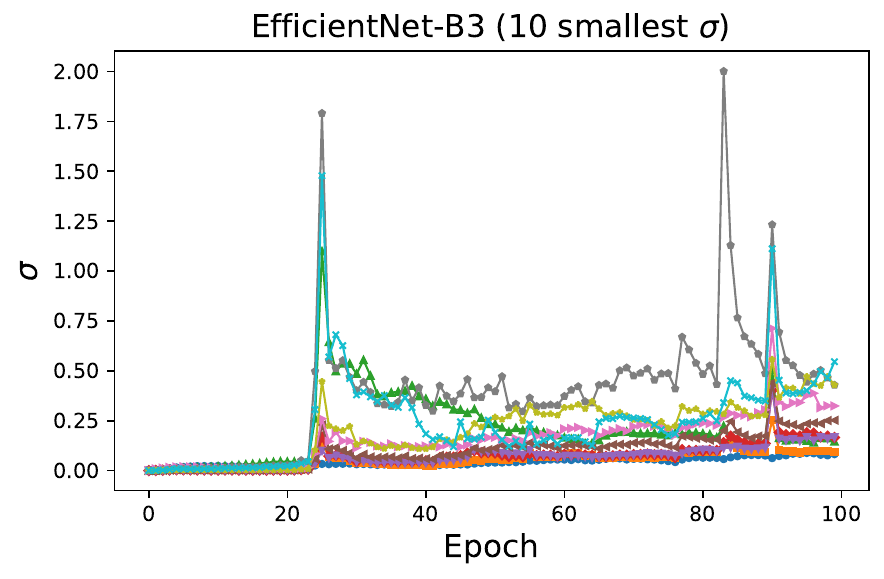}
        \end{minipage}
    \begin{minipage}[b]{0.23\textwidth} 
            \centering
            \includegraphics[width=\textwidth]{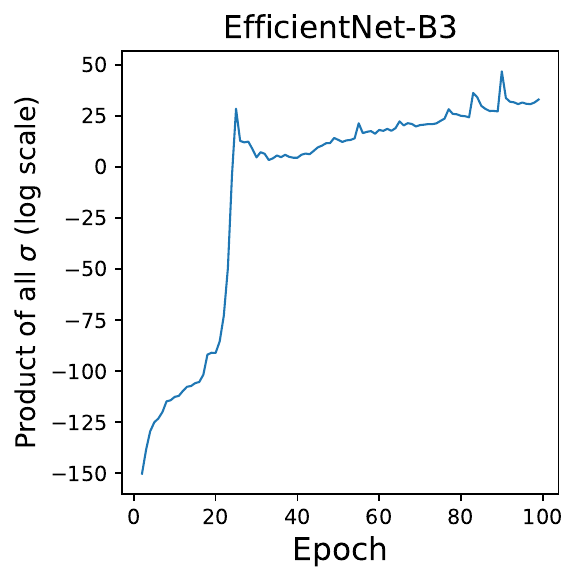} 
        \end{minipage}
\caption{Evolution of weight norms and input variances for some layers in  EfficientNet-B3 trained on the CIFAR-10 dataset. The input variance $(\sigma^2)$ is computed over mini-batches before  each $\BN$ layer. } 
\label{fig:preBN-variance-Norm-EfficientNet}
\end{figure*}

Figure \ref{fig:preBN-variance-Norm-EfficientNet} reports some behaviors of EfficientNet-B3 along the training process. We observe that variance at some initial steps has  small values. However, those variances and norms of the weight matrices consistently increase as training the model with more mini-batches. After the training, those quantities can be significantly large, leading to huge products of norms/variances. This is consistent with the behaviors of ResNet18 and the 10-layer ReLU network in the main paper.

\vskip 0.2in

\bibliography{ann,normalizers,other}

@article{cisneros25WN,
  title={Optimization and Generalization Guarantees for Weight Normalization},
  author={Cisneros-Velarde, Pedro and Chen, Zhijie and Koyejo, Sanmi and Banerjee, Arindam},
  journal={Transactions on Machine Learning Research},
    year={2025}
}

@article{xu2012robustnessGeneralize,
  title={Robustness and generalization},
  author={Xu, Huan and Mannor, Shie},
  journal={Machine learning},
  volume={86},
  number={3},
  pages={391--423},
  year={2012},
  publisher={Springer}
}

@inproceedings{frans25Shortcut,
  title={One Step Diffusion via Shortcut Models},
  author={Frans, Kevin and Hafner, Danijar and Levine, Sergey and Abbeel, Pieter},
  booktitle={The Thirteenth International Conference on Learning Representations},
    year={2025}
}

@inproceedings{gu2024mamba,
  title={Mamba: Linear-time sequence modeling with selective state spaces},
  author={Gu, Albert and Dao, Tri},
  booktitle={First Conference on Language Modeling},
  year={2024}
}

@inproceedings{behrouz2025nested,
  title={Nested learning: The illusion of deep learning architectures},
  author={Behrouz, Ali and Razaviyayn, Meisam and Zhong, Peilin and Mirrokni, Vahab},
  booktitle={The Thirty-ninth Annual Conference on Neural Information Processing Systems},
  year={2025}
}

@inproceedings{peebles2023DiT,
  title={Scalable diffusion models with transformers},
  author={Peebles, William and Xie, Saining},
  booktitle={Proceedings of the IEEE/CVF International Conference on Computer Vision},
  pages={4195--4205},
  year={2023}
}

@inproceedings{ma2024sit,
  title={Sit: Exploring flow and diffusion-based generative models with scalable interpolant transformers},
  author={Ma, Nanye and Goldstein, Mark and Albergo, Michael S and Boffi, Nicholas M and Vanden-Eijnden, Eric and Xie, Saining},
  booktitle={European Conference on Computer Vision},
  pages={23--40},
  year={2024},
  organization={Springer}
}

@misc{openai2024gpt4,
      title={GPT-4 Technical Report}, 
      author={OpenAI and Josh Achiam and Steven Adler and Sandhini Agarwal and Lama Ahmad and Ilge Akkaya and Florencia Leoni Aleman and Diogo Almeida and Janko Altenschmidt and Sam Altman and Shyamal Anadkat and Red Avila and Igor Babuschkin and Suchir Balaji and Valerie Balcom and Paul Baltescu and Haiming Bao and Mohammad Bavarian and Jeff Belgum and Irwan Bello and Jake Berdine and Gabriel Bernadett-Shapiro and Christopher Berner and Lenny Bogdonoff and Oleg Boiko and Madelaine Boyd and Anna-Luisa Brakman and Greg Brockman and Tim Brooks and Miles Brundage and Kevin Button and Trevor Cai and Rosie Campbell and Andrew Cann and Brittany Carey and Chelsea Carlson and Rory Carmichael and Brooke Chan and Che Chang and Fotis Chantzis and Derek Chen and Sully Chen and Ruby Chen and Jason Chen and Mark Chen and Ben Chess and Chester Cho and Casey Chu and Hyung Won Chung and Dave Cummings and Jeremiah Currier and Yunxing Dai and Cory Decareaux and Thomas Degry and Noah Deutsch and Damien Deville and Arka Dhar and David Dohan and Steve Dowling and Sheila Dunning and Adrien Ecoffet and Atty Eleti and Tyna Eloundou and David Farhi and Liam Fedus and Niko Felix and Simón Posada Fishman and Juston Forte and Isabella Fulford and Leo Gao and Elie Georges and Christian Gibson and Vik Goel and Tarun Gogineni and Gabriel Goh and Rapha Gontijo-Lopes and Jonathan Gordon and Morgan Grafstein and Scott Gray and Ryan Greene and Joshua Gross and Shixiang Shane Gu and Yufei Guo and Chris Hallacy and Jesse Han and Jeff Harris and Yuchen He and Mike Heaton and Johannes Heidecke and Chris Hesse and Alan Hickey and Wade Hickey and Peter Hoeschele and Brandon Houghton and Kenny Hsu and Shengli Hu and Xin Hu and Joost Huizinga and Shantanu Jain and Shawn Jain and Joanne Jang and Angela Jiang and Roger Jiang and Haozhun Jin and Denny Jin and Shino Jomoto and Billie Jonn and Heewoo Jun and Tomer Kaftan and Łukasz Kaiser and Ali Kamali and Ingmar Kanitscheider and Nitish Shirish Keskar and Tabarak Khan and Logan Kilpatrick and Jong Wook Kim and Christina Kim and Yongjik Kim and Jan Hendrik Kirchner and Jamie Kiros and Matt Knight and Daniel Kokotajlo and Łukasz Kondraciuk and Andrew Kondrich and Aris Konstantinidis and Kyle Kosic and Gretchen Krueger and Vishal Kuo and Michael Lampe and Ikai Lan and Teddy Lee and Jan Leike and Jade Leung and Daniel Levy and Chak Ming Li and Rachel Lim and Molly Lin and Stephanie Lin and Mateusz Litwin and Theresa Lopez and Ryan Lowe and Patricia Lue and Anna Makanju and Kim Malfacini and Sam Manning and Todor Markov and Yaniv Markovski and Bianca Martin and Katie Mayer and Andrew Mayne and Bob McGrew and Scott Mayer McKinney and Christine McLeavey and Paul McMillan and Jake McNeil and David Medina and Aalok Mehta and Jacob Menick and Luke Metz and Andrey Mishchenko and Pamela Mishkin and Vinnie Monaco and Evan Morikawa and Daniel Mossing and Tong Mu and Mira Murati and Oleg Murk and David Mély and Ashvin Nair and Reiichiro Nakano and Rajeev Nayak and Arvind Neelakantan and Richard Ngo and Hyeonwoo Noh and Long Ouyang and Cullen O'Keefe and Jakub Pachocki and Alex Paino and Joe Palermo and Ashley Pantuliano and Giambattista Parascandolo and Joel Parish and Emy Parparita and Alex Passos and Mikhail Pavlov and Andrew Peng and Adam Perelman and Filipe de Avila Belbute Peres and Michael Petrov and Henrique Ponde de Oliveira Pinto and Michael and Pokorny and Michelle Pokrass and Vitchyr H. Pong and Tolly Powell and Alethea Power and Boris Power and Elizabeth Proehl and Raul Puri and Alec Radford and Jack Rae and Aditya Ramesh and Cameron Raymond and Francis Real and Kendra Rimbach and Carl Ross and Bob Rotsted and Henri Roussez and Nick Ryder and Mario Saltarelli and Ted Sanders and Shibani Santurkar and Girish Sastry and Heather Schmidt and David Schnurr and John Schulman and Daniel Selsam and Kyla Sheppard and Toki Sherbakov and Jessica Shieh and Sarah Shoker and Pranav Shyam and Szymon Sidor and Eric Sigler and Maddie Simens and Jordan Sitkin and Katarina Slama and Ian Sohl and Benjamin Sokolowsky and Yang Song and Natalie Staudacher and Felipe Petroski Such and Natalie Summers and Ilya Sutskever and Jie Tang and Nikolas Tezak and Madeleine B. Thompson and Phil Tillet and Amin Tootoonchian and Elizabeth Tseng and Preston Tuggle and Nick Turley and Jerry Tworek and Juan Felipe Cerón Uribe and Andrea Vallone and Arun Vijayvergiya and Chelsea Voss and Carroll Wainwright and Justin Jay Wang and Alvin Wang and Ben Wang and Jonathan Ward and Jason Wei and CJ Weinmann and Akila Welihinda and Peter Welinder and Jiayi Weng and Lilian Weng and Matt Wiethoff and Dave Willner and Clemens Winter and Samuel Wolrich and Hannah Wong and Lauren Workman and Sherwin Wu and Jeff Wu and Michael Wu and Kai Xiao and Tao Xu and Sarah Yoo and Kevin Yu and Qiming Yuan and Wojciech Zaremba and Rowan Zellers and Chong Zhang and Marvin Zhang and Shengjia Zhao and Tianhao Zheng and Juntang Zhuang and William Zhuk and Barret Zoph},
      year={2024},
      eprint={2303.08774},
      url={https://arxiv.org/abs/2303.08774}, 
}

@misc{abdin2024phi4,
      title={Phi-4 Technical Report}, 
      author={Marah Abdin and Jyoti Aneja and Harkirat Behl and Sébastien Bubeck and Ronen Eldan and Suriya Gunasekar and Michael Harrison and Russell J. Hewett and Mojan Javaheripi and Piero Kauffmann and James R. Lee and Yin Tat Lee and Yuanzhi Li and Weishung Liu and Caio C. T. Mendes and Anh Nguyen and Eric Price and Gustavo de Rosa and Olli Saarikivi and Adil Salim and Shital Shah and Xin Wang and Rachel Ward and Yue Wu and Dingli Yu and Cyril Zhang and Yi Zhang},
      year={2024},
      eprint={2412.08905},
      url={https://arxiv.org/abs/2412.08905}, 
}

@inproceedings{zhang2019RMSnorm,
  title={Root mean square layer normalization},
  author={Zhang, Biao and Sennrich, Rico},
  journal={Advances in Neural Information Processing Systems},
  volume={32},
  year={2019}
}

@inproceedings{arora19BN,
  title={Theoretical Analysis of Auto Rate-Tuning by Batch Normalization},
  author={Arora, Sanjeev and Li, Zhiyuan and Lyu, Kaifeng},
  booktitle={International Conference on Learning Representations},
    year={2019}
}

@article{mueller2023normalization,
  title={Normalization layers are all that sharpness-aware minimization needs},
  author={Mueller, Maximilian and Vlaar, Tiffany and Rolnick, David and Hein, Matthias},
  journal={Advances in Neural Information Processing Systems},
  volume={36},
  pages={69228--69252},
  year={2023}
}

@inproceedings{frankle21BN,
  title={Training BatchNorm and Only BatchNorm: On the Expressive Power of Random Features in CNNs},
  author={Frankle, Jonathan and Schwab, David J and Morcos, Ari S},
  booktitle={International Conference on Learning Representations},
        year={2021}
}

@inproceedings{li20BN,
  title={An Exponential Learning Rate Schedule for Deep Learning},
  author={Li, Zhiyuan and Arora, Sanjeev},
  booktitle={International Conference on Learning Representations},
      year={2020}
}

@inproceedings{burkholzbatch24,
  title={Batch normalization is sufficient for universal function approximation in CNNs},
  author={Burkholz, Rebekka},
  booktitle={The Twelfth International Conference on Learning Representations},
    year={2024}
}

@inproceedings{cai2019quantitative,
  title={A quantitative analysis of the effect of batch normalization on gradient descent},
  author={Cai, Yongqiang and Li, Qianxiao and Shen, Zuowei},
  booktitle={International Conference on Machine Learning},
  pages={882--890},
  year={2019},
  organization={PMLR}
}

@article{karakida2019normalization,
  title={The normalization method for alleviating pathological sharpness in wide neural networks},
  author={Karakida, Ryo and Akaho, Shotaro and Amari, Shun-ichi},
  journal={Advances in Neural Information Processing Systems},
  volume={32},
  year={2019}
}

@inproceedings{lyu22understandingNorm,
  title={Understanding the Generalization Benefit of Normalization Layers: Sharpness Reduction},
  author={Lyu, Kaifeng and Li, Zhiyuan and Arora, Sanjeev},
  booktitle={Advances in Neural Information Processing Systems},
  year={2022}
}

@inproceedings{de2020BN,
  title={Batch Normalization Biases Residual Blocks Towards the Identity Function in Deep Networks},
  author={De, Soham and Smith, Sam},
  booktitle={Advances in Neural Information Processing Systems},
  volume={33},
  year={2020}
}

@article{ioffe2017batchRe,
  title={Batch Renormalization: Towards Reducing Minibatch Dependence in Batch-Normalized Models},
  author={Ioffe, Sergey},
  journal={Advances in Neural Information Processing Systems},
  volume={30},
  pages={1945--1953},
  year={2017}
}

@inproceedings{kohler2019convergenceBN,
  title={Exponential convergence rates for batch normalization: The power of length-direction decoupling in non-convex optimization},
  author={Kohler, Jonas and Daneshmand, Hadi and Lucchi, Aurelien and Hofmann, Thomas and Zhou, Ming and Neymeyr, Klaus},
  booktitle={The 22nd International Conference on Artificial Intelligence and Statistics},
  pages={806--815},
  year={2019},
  organization={PMLR}
}

@inproceedings{zhang2019fixup,
  title={Fixup Initialization: Residual Learning Without Normalization},
  author={Zhang, Hongyi and Dauphin, Yann N and Ma, Tengyu},
  booktitle={International Conference on Learning Representations},
  year={2019}
}

@inproceedings{santurkar2018BNorm,
  title={How does batch normalization help optimization?},
  author={Santurkar, Shibani and Tsipras, Dimitris and Ilyas, Andrew and Madry, Aleksander},
  booktitle={Advances in Neural Information Processing Systems (NeurIPS)},
  pages={2488--2498},
  year={2018}
}

@inproceedings{shao2020normalization,
  title={Is normalization indispensable for training deep neural network?},
  author={Shao, Jie and Hu, Kai and Wang, Changhu and Xue, Xiangyang and Raj, Bhiksha},
  booktitle={Advances in Neural Information Processing Systems},
  volume={33},
  year={2020}
}

@article{awais2020revisitingBN,
  title={Revisiting internal covariate shift for batch normalization},
  author={Awais, Muhammad and Iqbal, Md Tauhid Bin and Bae, Sung-Ho},
  journal={IEEE Transactions on Neural Networks and Learning Systems},
  year={2020},
  publisher={IEEE}
}

@inproceedings{luo2019UnderstandingBN,
  title={Towards Understanding Regularization in Batch Normalization},
  author={Luo, Ping and Wang, Xinjiang and Shao, Wenqi and Peng, Zhanglin},
  booktitle={International Conference on Learning Representations},
  year={2019}
}

@article{wu2020WN,
  title={Implicit Regularization and Convergence for Weight Normalization},
  author={Wu, Xiaoxia and Dobriban, Edgar and Ren, Tongzheng and Wu, Shanshan and Li, Zhiyuan and Gunasekar, Suriya and Ward, Rachel and Liu, Qiang},
  journal={Advances in Neural Information Processing Systems},
  volume={33},
  year={2020}
}

@inproceedings{bjorck2018understandingBN,
  title={Understanding batch normalization},
  author={Bjorck, Johan and Gomes, Carla and Selman, Bart and Weinberger, Kilian Q},
  booktitle={Advances in Neural Information Processing Systems (NeurIPS)},
  pages={7705--7716},
  year={2018}
}

@article{ulyanov2016INorm,
  title={Instance normalization: The missing ingredient for fast stylization},
  author={Ulyanov, Dmitry and Vedaldi, Andrea and Lempitsky, Victor},
  journal={arXiv preprint arXiv:1607.08022},
  year={2016}
}

@article{wu2020GNorm,
  title={Group Normalization},
  author={Wu, Yuxin and He, Kaiming},
  journal={International Journal of Computer Vision},
  volume={128},
  number={3},
  pages={742--755},
  year={2020},
  publisher={Springer}
}

@inproceedings{ioffe2015BNorm,
  title={Batch normalization: Accelerating deep network training by reducing internal covariate shift},
  author={Ioffe, Sergey and Szegedy, Christian},
  booktitle={International Conference on Machine Learning},
  pages={448--456},
  year={2015},
  organization={PMLR}
}

@article{ba2016LNorm,
  title={Layer normalization},
  author={Ba, Jimmy Lei and Kiros, Jamie Ryan and Hinton, Geoffrey E},
  journal={arXiv preprint arXiv:1607.06450},
  year={2016}
}

@article{than2025LocalRobustness,
  title={Gentle Local Robustness implies Generalization},
  author={Than, Khoat and Phan, Dat and Vu, Giang},
  journal={Machine Learning},
  year={2025},
  publisher={Springer}
}

@inproceedings{davis2022GD,
  title={A gradient sampling method with complexity guarantees for lipschitz functions in high and low dimensions},
  author={Davis, Damek and Drusvyatskiy, Dmitriy and Lee, Yin Tat and Padmanabhan, Swati and Ye, Guanghao},
  booktitle={Advances in Neural Information Processing Systems},
  volume={35},
  pages={6692--6703},
  year={2022}
}

@inproceedings{tian2022finite,
  title={On the finite-time complexity and practical computation of approximate stationarity concepts of lipschitz functions},
  author={Tian, Lai and Zhou, Kaiwen and So, Anthony Man-Cho},
  booktitle={International Conference on Machine Learning},
  pages={21360--21379},
  year={2022},
  organization={PMLR}
}

@inproceedings{zhang2020OPTcomplexity,
  title={Complexity of finding stationary points of nonconvex nonsmooth functions},
  author={Zhang, Jingzhao and Lin, Hongzhou and Jegelka, Stefanie and Sra, Suvrit and Jadbabaie, Ali},
  booktitle={International Conference on Machine Learning},
  pages={11173--11182},
  year={2020}
}

@article{bubeck2015convex,
  title={Convex optimization: Algorithms and complexity},
  author={Bubeck, S{\'e}bastien and others},
  journal={Foundations and Trends{\textregistered} in Machine Learning},
  volume={8},
  number={3-4},
  pages={231--357},
  year={2015},
  publisher={Now Publishers, Inc.}
}

@article{hou2023instanceGen,
  title={Instance-Dependent Generalization Bounds via Optimal Transport},
  author={Hou, Songyan and Kassraie, Parnian and Kratsios, Anastasis and Rothfuss, Jonas and Krause, Andreas},
  journal={The Journal of Machine Learning Research},
  year={2023}
}

@InProceedings{kawaguchi22RobustGen,
  title = 	 {Robustness Implies Generalization via Data-Dependent Generalization Bounds},
  author =       {Kawaguchi, Kenji and Deng, Zhun and Luh, Kyle and Huang, Jiaoyang},
  booktitle = 	 {Proceedings of the 39th International Conference on Machine Learning},
  pages = 	 {10866--10894},
  year = 	 {2022},
  editor = 	 {Chaudhuri, Kamalika and Jegelka, Stefanie and Song, Le and Szepesvari, Csaba and Niu, Gang and Sabato, Sivan},
  volume = 	 {162},
  series = 	 {Proceedings of Machine Learning Research},
  publisher =    {PMLR}
}

@inproceedings{neyshabur2018SpectralMarginDNN,
  title={A PAC-Bayesian Approach to Spectrally-Normalized Margin Bounds for Neural Networks},
  author={Neyshabur, Behnam and Bhojanapalli, Srinadh and Srebro, Nathan},
  booktitle={International Conference on Learning Representations},
  year={2018}
}

@article{bartlett2017SpectralMarginDNN,
  title={Spectrally-normalized margin bounds for neural networks},
  author={Bartlett, Peter L and Foster, Dylan J and Telgarsky, Matus J},
  journal={Advances in Neural Information Processing Systems},
  volume={30},
  pages={6240--6249},
  year={2017}
}

@article{golowich2020RC,
  title={Size-independent sample complexity of neural networks},
  author={Golowich, Noah and Rakhlin, Alexander and Shamir, Ohad},
  journal={Information and Inference: A Journal of the IMA},
  volume={9},
  number={2},
  pages={473--504},
  year={2020},
  publisher={Oxford University Press}
}

@inproceedings{arora2021dropout,
  title={Dropout: Explicit forms and capacity control},
  author={Arora, Raman and Bartlett, Peter and Mianjy, Poorya and Srebro, Nathan},
  booktitle={International Conference on Machine Learning},
  pages={351--361},
  year={2021},
  organization={PMLR}
}

\end{document}